\documentclass{article}

\usepackage[final]{nips_2016}
\usepackage[T1]{fontenc}
\usepackage[x11names]{xcolor}
\usepackage{amsfonts}
\usepackage{amsmath}
\usepackage{amsthm}
\usepackage{graphicx}
\usepackage[colorlinks=false, pdfborder={0 0 0}]{hyperref}
\usepackage{tikz}
\usepackage{url}            
\usepackage{booktabs}       
\usepackage{nicefrac}       
\usepackage{microtype}      

\newcommand*{\Reals}{\mathbb{R}}

\renewcommand*{\O}{\mathcal{O}}

\newcommand*{\corr}{\mathrm{corr}}
\newcommand*{\e}[1]{\cdot 10^{#1}}
\newcommand*{\p}[2]{\frac{\partial #1}{\partial #2}}
\newcommand*{\sgn}{\mathrm{sgn}}
\newcommand*{\var}{\mathrm{Var}}

\newcommand*{\cauchy}{\mathrm{Cauchy}}

\renewcommand{\v}{\mathbf{v}}
\newcommand{\w}{\mathbf{w}}
\newcommand{\x}{\mathbf{x}}
\newcommand{\y}{\mathbf{y}}

\newcommand*{\vbar}{\bar{v}}
\newcommand*{\dx}{\Delta x}

\newcommand{\stack}[2]{\genfrac{}{}{0pt}{}{#1}{#2}}

\newtheorem{theorem}{Theorem}
\newtheorem{lemma}{Lemma}
\newtheorem{corollary}{Corollary}
\newtheorem{conjecture}[theorem]{Conjecture}

\theoremstyle{definition}
\newtheorem*{definition}{Definition}

\theoremstyle{remark}
\newtheorem*{remark}{Remark}

\bibpunct{(}{)}{;}{a}{}{,}

\usetikzlibrary{arrows, positioning}
\tikzstyle{activation} = [draw, rectangle, fill=green!20, minimum height=6mm, minimum width=9mm]
\tikzstyle{combo} = [draw, rectangle, fill=blue!20, minimum height=1cm, minimum width=1.2cm, font=\footnotesize]
\tikzstyle{knot} = [draw, fill=blue, blue, circle, minimum width=1mm, inner sep=0pt]
\tikzstyle{neuron} = [draw, rectangle, fill=red!20, minimum height=6mm, minimum width=6mm]
\tikzstyle{scalar} = [minimum height=5.5mm, minimum width=5mm]
\tikzstyle{subfig} = [font=\small, inner sep=0pt, anchor=north west]
\tikzstyle{xtick} = [pos=0, below, black]

\title{Knots in random neural networks}
\author{%
    Kevin K. Chen \\
    Institute for Defense Analyses \\
    Center for Communications Research - La Jolla \\
    San Diego, CA 92121 \\
    \texttt{kkchen@ccrwest.org} \\
    \And
    Anthony C. Gamst \\
    Institute for Defense Analyses \\
    Center for Communications Research - La Jolla \\
    San Diego, CA 92121 \\
    \texttt{acgamst@ccrwest.org} \\
    \And
    Alden K. Walker \\
    Institute for Defense Analyses \\
    Center for Communications Research - La Jolla \\
    San Diego, CA 92121 \\
    \texttt{akwalke@ccrwest.org}%
}

\begin{document}

\maketitle

\begin{abstract}
    The weights of a neural network are typically initialized at random, and one can think of the functions produced by such a network as having been generated by a prior over some function space.
    Studying random networks, then, is useful for a Bayesian understanding of the network evolution in early stages of training.
    In particular, one can investigate why neural networks with huge numbers of parameters do not immediately overfit.
    We analyze the properties of random scalar-input feed-forward rectified linear unit architectures, which are random linear splines.
    With weights and biases sampled from certain common distributions, empirical tests show that the number of knots in the spline produced by the network is equal to the number of neurons, to very close approximation.
    We describe our progress towards a completely analytic explanation of this phenomenon.
    In particular, we show that random single-layer neural networks are equivalent to integrated random walks with variable step sizes.
    That each neuron produces one knot on average is equivalent to the associated integrated random walk having one zero crossing on average.
    We explore how properties of the integrated random walk, including the step sizes and initial conditions, affect the number of crossings.
    The number of knots in random neural networks can be related to the behavior of extreme learning machines, but it also establishes a prior preventing optimizers from immediately overfitting to noisy training data.
\end{abstract}

\section{Introduction}

In recent years, neural networks have experienced a resurgence in practical data science applications, and have been responsible for improvements in these applications by leaps and bounds.
The power of neural networks is in part an effect of the well-known universal approximation theorems \citep{CybenkoMCSS89,HornikNN91,HornikNN89,SonodaACHA}, which essentially allow neural networks to model any vector-valued continuous function arbitrary well.
Despite the great success stories attributed to neural networks, however, the analytical understanding of their successes have not been understood as well as their applicability and use in specific applications.
Even the universal approximation theorems do not specify the rate that neural networks train on data, given the network size and various training data parameters.

In a companion paper \citep{ccr86994}, we showed that neural networks with commonly-employed rectified linear unit activation functions are simply linear splines.
Thus, one of many ways to measure the \emph{complexity} or \emph{expressivity} of a neural network is to count the number of spline knots or linear pieces in the entire domain of the neural network \citep{MontufarNIPS14, Pascanu14, Raghu16, Arora16}.
Knowledge of the maximum or expected number of knots in a neural network of a given width (i.e., neurons per layer) and depth (i.e., layers) is useful in neural network design.
If a designer had \emph{a priori} knowledge of the desired model complexity, then he could ensure that the neural network size is sufficient large.

Whereas we have previously derived an upper bound on the number of knots in neural networks \cite{ccr86994}, we now explore the number of knots in networks with random weights and biases.
This type of analysis is admittedly uncommon, and the motivations unintuitive.
Nevertheless, random neural networks provide a better venue for rigorous statistical analysis than application-specific networks.
Furthermore, there are at least three ways that the behavior of random neural networks can ultimately be related back to that of neural networks trained on actual data.

First, random neural networks are more representative of ``average-case'' neural networks than the ``best-case'' networks meeting the upper bound on the number of knots.
We showed that to construct a neural network with the maximum number of knots, every affine transformation in every layer except the output must be a maximally high-wavenumber sawtooth wave \cite{ccr86994}.
Such a construction is useful in providing a ``brick-wall'' limit on the expressivity of neural networks, but is extremely unlikely to be encountered in practice.
In fact, in a companion paper~\cite{WalkerSCAMP16}, we show that optimizers tend to learn low-wavenumber content before high-wavenumber content in underlying functions, and that very high wavenumbers may never be fully learned.
On the other hand, random neural networks are not artificially constructed to have maximally high wavenumbers, and---at least roughly speaking---are more representative of real data.

Second, we will later show that even when neural networks are trained on actual data, the weights and biases are typically random or close to random in the early stages of the training.
Therefore, understanding the behavior of random neural networks is important for analyzing the process of optimizing neural networks.
In particular, commonly-employed neural network optimizers typically do not immediately overfit neural networks to noisy training data.
It will become apparent that random optimization establishes a prior on the function space on which trained neural networks reside, and that the smoothness of random neural networks prevents overfitting.

Finally, random weights are a crucial element in the variant of neural networks known as extreme learning machines \cite{HuangNeurocomputing06}.
In extreme learning machines, input weights are chosen randomly, and output weights following a single hidden layer of neurons are solved by least-squares.
Extreme learning machines are very simple and can be trained easily.
They have been employed successfully in simple problems \citep{HuangNN15}, and represent a real-life use of random neural networks.
Further improvements may be possible, for instance, by designing a large number of random networks and applying $L^1$ regularization to combine them.

This paper is fairly exploratory in nature.
As such, it covers a few different topics related to random neural networks.
The key focus of this paper is the empirical result that the number of knots in a scalar-input random neural network is simply the number of neurons in the network, to very close approximation.
This is in contrast to the upper bound, which for $n$ neurons in each of $l$ layers is approximately $n^l$ \citep{ccr86994}.
Although we do not rigorously derive the expected number of knots in random neural networks, we outline our progress toward this objective.
Complete results are given where they are available, but many other aspects are more conjectural, empirical, or otherwise open for further exploration.

The topics and sections of this paper are as follows.
In \autoref{sec:neural-nets}, we briefly review the neural network architecture.
The key empirical result on the number of knots in random neural networks is given in \autoref{sec:empirical}, though further numerical experiments are discussed later as they become relevant.
Next, \autoref{sec:analysis} analyzes the behavior of random neural networks from a few different perspectives.
We consider the relation between knots in neural networks and roots of the affine transformations within neurons, and we analyze the effects of different probability distributions on the number of knots.
Afterwards, we draw an equivalence between random neural networks and variable-step-size integrated random walks in \autoref{sec:integrated-random-walk}.
We demonstrate an equivalence between the number of knots per neuron and the number of zero crossings in the associated integrated random walk \citep{DenisovAIHPPS15,GroeneboomAP99,SinaiTMP92}.
The majority of this section explores how properties of the associated integrated random walk, such as the step sizes and initial conditions, are related to the number of zero crossings.
\autoref{sec:early-training} then relates random neural networks to the early stages of neural network training, and discusses how the weight and bias distributions we explore in this paper actually appear during optimization processes.
Finally, we summarize our results and comment on directions for future research in \autoref{sec:conclusion}.

\section{Brief description of neural networks}
\label{sec:neural-nets}

In this section, we quickly review the basic definitions and properties of the neural network architecture.
We intentionally keep this section concise.
A more detailed description and accompanying figures can be found in a companion paper \citep{ccr86994}, and an overview of neural networks and other machine learning techniques can be found in \citet{KnoxML}.

For this paper, we will use the commonly employed rectified linear unit $\sigma(x) = \max(0, x)$ as the nonlinear activation function in our neural networks.
To construct a single-layer $\Reals^q \to \Reals^p$ neural network with $n$ neurons, we select input weights $\w_{1k} \in \Reals^q$ and input biases $b_{1k} \in \Reals$ for $k = 1, \dots, n$, and output weights $\w_{2k} \in \Reals^n$ and output biases $b_{2k} \in \Reals$ for $k = 1, \dots, p$.
Using the shorthand notation $\v := [v_1 \; \cdots \; v_n]$ to indicate the hidden layer outputs, the single-layer neural network maps $\x \in \Reals^q$ to $\y = [\eta_1 \; \cdots \; \eta_p] \in \Reals^p$ by
\begin{subequations}
    \label{eq:shallow}
    \begin{align}
        \label{eq:shallow-input}
        v_k &:= \sigma(\w_{1k} \cdot \x + b_{1k}), \quad
        k = 1, \dots, n, \\
        \label{eq:shallow-output}
        \eta_k &:= \w_{2k} \cdot \v + b_{2k}, \quad
        k = 1, \dots, p.
    \end{align}
\end{subequations}

A deep neural network is conceptually identical, except that a serial chain of hidden layers exists between the input and the output.
Now, we consider $l$ hidden layers, with $n_i$ neurons in layer $i = 1, \dots, l$.
As before, we assign input weights $\w_{1k} \in \Reals^q$ and input biases $b_{1k} \in \Reals$, now for $k = 1, \dots, n_1$.
Furthermore, we choose additional weights $\w_{ik} \in \Reals^{n_{i-1}}$ and biases $b_{ik} \in \Reals$ for layers $i = 2, \dots, l$, and for $k = 1, \dots, n_i$.
Finally, we assign output weights $\w_{l+1,k} \in \Reals^{n_l}$ and output biases $b_{l+1,k} \in \Reals$ for $k = 1, \dots, p$.
Let us use the shorthand notation $\v_i := [v_{i1} \; \cdots \; v_{i n_i}]$.  The deep neural network is then given by
\begin{subequations}
    \label{eq:deep}
    \begin{align}
        \label{eq:deep-input}
        v_{1k} &:= \sigma(\w_{1k} \cdot \x + b_{1k}), \quad
        k = 1, \dots, n_1 \\
        \label{eq:deep-hidden}
        v_{ik} &:= \sigma(\w_{ik} \cdot \v_{i-1} + b_{ik}), \quad
        i = 2, \dots, l, \quad
        k = 1, \dots, n_i \\
        \label{eq:deep-output}
        \eta_k &:= \w_{l+1, k} \cdot \v_l + b_{l+1, k}, \quad
        k = 1, \dots, p.
    \end{align}
\end{subequations}

In \citet{ccr86994}, we showed that neural networks with rectified linear unit activations are linear splines---that is, they are piecewise linear functions with a finite number of pieces.
For $i = 1, \dots, l + 1$ and with $\v = \x$ or $\v_{i-1}$, every affine transformation $\w_{ik} \cdot \v + b_{ik}$ combines the knots in the scalar components of $\v$.
The application of $\sigma$ to $\w_{ik} \cdot \v + b_{ik}$ then retains knots at inputs $\x_j$ if $\w_{ik} \cdot \v(\x_j) + b_{ik} > 0$, and eliminates knots at inputs $\x_j$ if $\w_{ik} \cdot \v(\x_j) + b_{ik} < 0$.
Furthermore, wherever $\w_{ik} \cdot \v(\x) + b_{ik} = 0$, new knots are created at those values of $\x$.
Thus, there is a natural equivalence between knots in neural networks and roots of affine transformations: wherever a knot exists in a neural network, some neuron in some layer has a root there in its affine transformation.

Furthermore, as discussed in \citet{ccr86994}, the output dimension $p$ nominally has no effect on the number of knots in a neural network.
The output layer~(\ref{eq:shallow-output},~\ref{eq:deep-output}) does not contain rectified linear units, so all knots must be created in the hidden layers.
The only degenerate case that may arise is that any layer, including the output layer, may theoretically construct affine transformations such that $\v(\x)$ has a discontinuity in some gradient component, but $\w_{ik} \cdot \v(\x)$ does not.
We will not be concerned with this possibility in this paper, because such an occurrence almost surely will not happen if the weights and biases are randomly chosen from continuous distributions.

We will primarily consider the $\Reals \to \Reals$ single-layer neural network
\begin{equation}
    \label{eq:scalar-shallow}
    y(x) = \sum_{j=1}^n w_{2j} \sigma(w_{1j} x + b_{1j}) + b_2,
\end{equation}
which is obtained directly from~\eqref{eq:shallow} with a slight change in notation.
Such a simple model, admittedly, is not typically employed in practice; neural networks are commonly $\Reals^q \to \Reals^p$ functions in practice.
It does allow certain analyses, however, that would otherwise be far more difficult.
On one hand, our results apply immediately to output dimensions $p > 1$ for the aforementioned reasons.
On the other hand, the extension to input dimensions $q > 1$ introduces multidimensional polytopes, and the use of deep networks with $l > 1$ layers raises additional complexities that have not yet been considered.
The extension of our results to generic vector-valued deep neural networks remains a future objective.

\section{Empirical results}
\label{sec:empirical}

To motivate the discussion on random neural networks, we first show some examples of scalar-valued single-layer neural networks~\eqref{eq:scalar-shallow}.
\autoref{fig:example}
\begin{figure}[!t]
    \centering
    \includegraphics{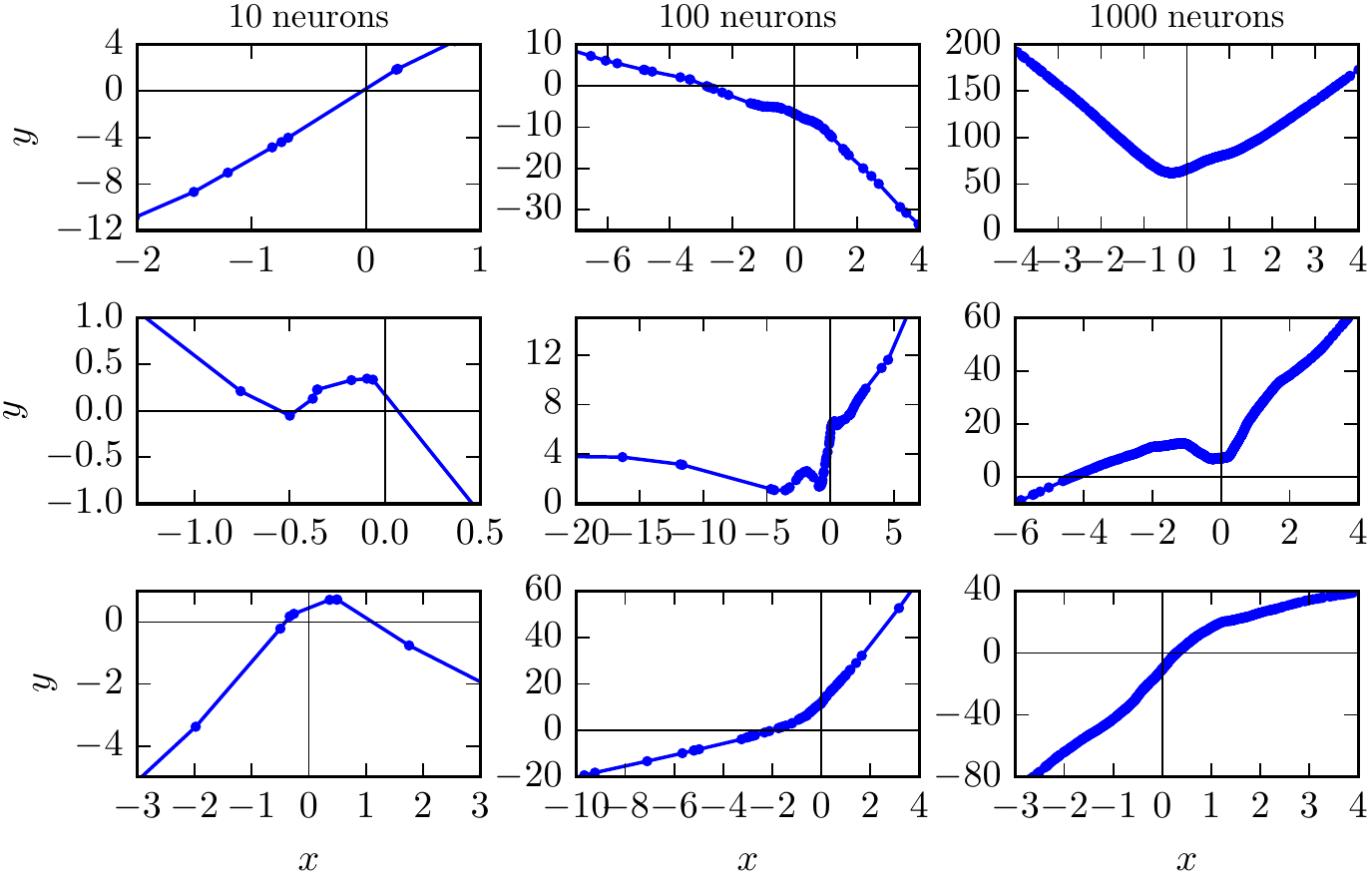}

    \caption{%
        Example random scalar-valued single-layer neural networks~\eqref{eq:scalar-shallow}, with $n = 10$ (left), 100 (middle), and 1,000 (right).
        Three examples for each size are shown in rows.
        Not all the knots are visible because the axes are scaled to highlight the behavior near the origin.%
    }
    \label{fig:example}
\end{figure}
shows three examples each with three different neural network widths $n$, where weights and biases are drawn independently from $N(0, 1)$.
As discussed in \citet{ccr86994}, the knot locations for this single-layer neural network model are simply $x_j = -b_{1j} / w_{1j}$ for $j = 1, \dots, n$.
Thus, with $n$ neurons in the single hidden layer, there are almost surely $n$ unique knots in the model.

With ten neurons in the random network (\autoref{fig:example}, left column), the linear spline is fairly coarse, as the model contains exactly ten knots.
With 100 and particularly with 1,000 neurons (\autoref{fig:example}, middle and right columns), however, the distribution of knots is somewhat dense.
In the latter case, the distribution is so dense that nearly every pixel in the right column plots contains a knot.
We observe that by and large, the random neural networks are qualitatively fairly smooth.
This topic will be related to integrated random walks and revisited in \autoref{sec:integrated-random-walk}.

Our chief numerical experiment, which motivates much of the remainder of this paper, is the following.
We consider $\Reals \to \Reals$ neural networks with $l = 1, \dots, 5$ hidden layers; in each, we consider a number of knots $n$ = 10, 20, 40, 60, 80, 100, 200 in every layer.
For every one of the 35 different neural network sizes, every scalar weight and bias was drawn independently from $U(-1, 1)$.
Then, for four such trials with this network size, the number of knots in the entire domain $x \in \Reals$ was counted by keeping track of the linear splines input into each neuron.

The number of knots $m$ in the resulting random neural networks is shown in \autoref{fig:experiment}.
\begin{figure}[!t]
    \centering
    \includegraphics{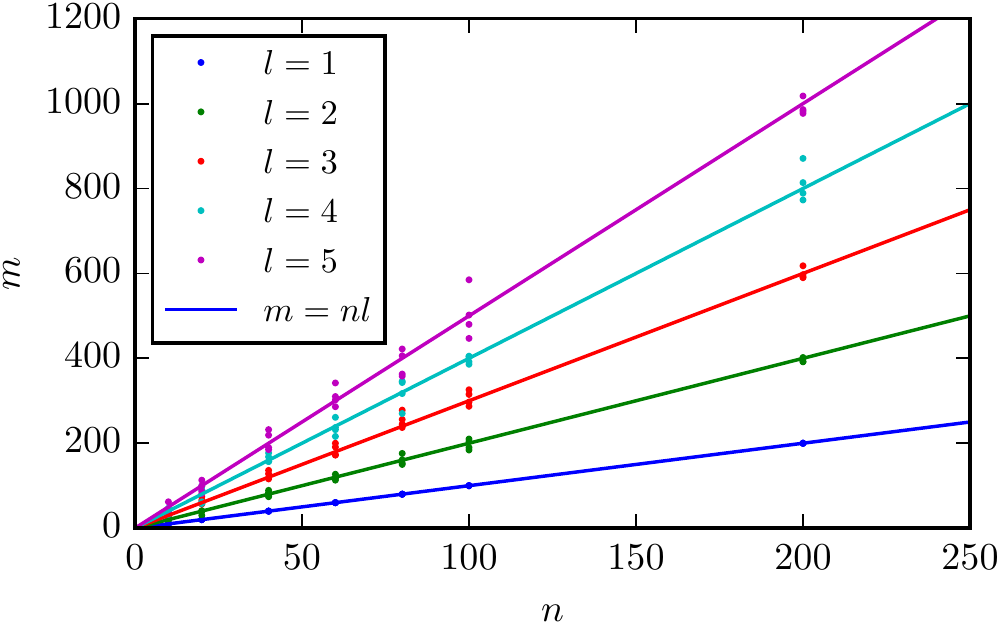}

    \caption{%
        The number of knots $m$ for different numbers of layers $l$ and neurons per layer $n$.
        Each dot shows a trial, and the lines $m = n l$ are also shown.%
    }
    \label{fig:experiment}
\end{figure}
The main empirical result is stated as follows.

\begin{conjecture}
    \label{conj:num-knots}

    In an $\Reals \to \Reals$ neural network with $l$ layers, $n$ neurons per layer, rectified linear unit activation functions in each neuron, and weights and biases drawn independently from $U(-1, 1)$, the number of knots in the input--output model is
    \begin{equation}
        \label{eq:num-knots}
        m \approx n l.
    \end{equation}
\end{conjecture}

That is, the number of knots is approximately equal to the number of neurons.
For this limited data set, the mean of $(m - n l) / (n l)$---the error in~\eqref{eq:num-knots}, normalized by the number of neurons---is $8.
5\e{-4}$, and the standard deviation is $9.
4\e{-2}$.
Although~\eqref{eq:num-knots} is an extraordinarily simple observation, the proof of this conjecture has been surprisingly evasive.
The main objective from this point forward is to seek an analytical explanation for~\eqref{eq:num-knots}.
We do not rigorously prove this result in this paper, but we detail our progress and discuss some expository ideas that may lead to a proof.

We also remark that for an $\Reals \to \Reals^p$ neural network, the number of scalar weights and biases is approximately $(p + 2) n + (l - 1) n^2$ \cite{ccr86994}.
Although for $l \ge 3$, this is much less than the tight upper bound on the number of knots, which is approximately $n^l$ \cite{ccr86994}, it is still typically much greater than the $n l$ knots observed in random neural networks.

\section{Analysis of random neural networks}
\label{sec:analysis}

As a first step in understanding the behavior exhibited in \autoref{fig:experiment}, we consider in \autoref{sec:knots-roots} how knots are created and eliminated in each hidden layer.
Next, in \autoref{sec:forward-facing} we review a transformation that allows the neural network~\eqref{eq:scalar-shallow} to be written so that all rectified linear units face the positive $x$ direction.
Using this transformation, \autoref{sec:distributions} discusses the number of knots in random neural networks for different probability distributions.

\subsection{Relating knots and roots}
\label{sec:knots-roots}

It was shown in \cite{ccr86994} and \autoref{sec:neural-nets} that each neuron can retain, eliminate, or create knots, depending on the value of the affine transformation $\w_{ik} \cdot \v_{i-1}(x) + b_{ik}$ to which the rectified linear unit is applied.
Regarding the collective behavior of neurons, one can further show the following.

\begin{lemma}
    \label{lem:none-eliminated}

    Consider a deep neural network with weights and biases independently selected from continuous distributions that are symmetric about zero.
    For hidden layer $i = 2, \dots, l$, in the limit that $n_i \to \infty$, every knot of $\v_{i-1}(x)$ is almost surely also a knot of $\v_i(x)$.
\end{lemma}

\begin{proof}
    The sketch of the proof is fairly simple.
    The output of neuron $k$ in hidden layer $i$ is given by~\eqref{eq:deep-hidden}.
    If the weights and biases are selected independently from continuous distributions that are symmetric about zero, then for each knot $x_j$ among the elements of $\v_{i-1}$, there is no preference in the sign of $\w_{ik} \cdot \v_{i-1}(x_j) + b_{ik}$, besides that it is almost surely not zero.
    Furthermore, there is a zero probability that this affine transformation will degenerately remove discontinuities, such that $\v_{i-1}$ has a knot at some $x_j$ but $\w_{ik} \cdot \v_{i-1}$ does not.
    Therefore, for each knot $x_j$, the probability that $\w_{ik} \cdot \v_{i-1}(x_j) + b_{ik} < 0$---and thus the knot $x_j$ is eliminated by the application of a rectified linear unit---is exactly $1 / 2$.

    Since each neuron's weights and biases are chosen independently, the probability that a knot is eliminated by a neuron's rectified linear unit is independent of all other neurons.
    Note that for a knot to be absent among the outputs of layer $i$, it must be eliminated by every neuron in that layer.
    The probability of such an occurrence is $2^{-n_i}$.
    Finally, if there are $m_{i-1}$ knots in $\v_{i-1}$, then the probability that all $m_{i-1}$ are retained in $\v_i$ is \begin{equation} P_\text{preserved} = (1 - 2^{-n_i})^{m_{i-1}}, \end{equation} which approaches unity as $n_i \to \infty$.
\end{proof}

\begin{corollary}
    \label{cor:none-eliminated-num-knots}
    Supposing that $n_i = n$ for all $i$, Lemma~\ref{lem:none-eliminated} holds even if $m_{i-1} = n (i - 1)$, as is empirically observed.
\end{corollary}

\begin{proof}
    In this case, the probability that all layer $i - 1$ knots are preserved in the layer $i$ outputs, for $n \to \infty$, is
    \begin{subequations}
        \begin{align}
            \lim_{n \to \infty} P_\text{preserved}
            &= \lim_{n \to \infty} (1 - 2^{-n})^{n (i - 1)} \\
            &= \lim_{n \to \infty} \exp(n (i - 1) \ln(1 - 2^{-n})).
        \end{align}
    \end{subequations}
    Using the change of variables given by $\epsilon = 2^{-n}$,
    \begin{equation}
        \lim_{n \to \infty} P_\text{preserved} = \lim_{\epsilon \to 0} \exp(-(i - 1) (\log_2 \epsilon) \ln(1 - \epsilon)).
    \end{equation}
    Expanding $\ln(1 - \epsilon)$ in a Taylor series,
    \begin{subequations}
        \begin{align}
            \lim_{n \to \infty} P_\text{preserved} &= \lim_{\epsilon
                \to 0} \exp\left((i - 1) (\log_2 \epsilon)
                \left(\epsilon + \O\left(\epsilon^2\right)\right)\right) \\
            &= 1.
        \end{align}
    \end{subequations}
\end{proof}

\begin{corollary}
    \label{cor:half-chance}

    With probability greater than $1 / 2$, all layer $i - 1$ knots are preserved in the layer $i$ outputs if
    \begin{equation}
        \label{eq:half-chance-series}
        n_i > \log_2 m_{i-1} - \log_2 \ln 2 + \O\left(m_{i-1}^{-1}\right).
    \end{equation}
    For the conditions in Corollary~\ref{cor:none-eliminated-num-knots}, the equivalent inequality is
    \begin{equation}
        \label{eq:half-chance-series-num-knots}
        n - \log_2 n + \O(n^{-1}) > \log_2(i - 1) - \log_2 \ln 2 + \O((i - 1)^{-1})
    \end{equation}
\end{corollary}

\begin{proof}
    By solving $P_\text{preserved} > 1 / 2$ for $n_i$, we obtain
    \begin{equation}
        \label{eq:half-chance}
        n_i > -\log_2 (1 - 2^{-1 / m_{i-1}}).
    \end{equation}
    Equation~\eqref{eq:half-chance-series} is then obtained by computing the series solution for~\eqref{eq:half-chance} with large $m_{i-1}$.
    Equation~\eqref{eq:half-chance-series-num-knots} is then derived from~\eqref{eq:half-chance-series} by setting $m_{i-1} = n (i - 1)$.
\end{proof}

Lemma~\ref{lem:none-eliminated} and the corollaries have an important implication in the analysis of knots in random neural networks.
Essentially, if $n_i$ is large for all $i = 2, \dots, l$---which is typically the case---then no knots are expected to be eliminated.
In fact, Corollary~\ref{cor:half-chance} shows that $n_i$ does not even need to be that large: it only needs to scale logarithmically by $m_{i-1}$ or $i - 1$.
Thus, for decently sized $n_i$, the number of knots in the input--output relation of the neural network is equal to the number of knots created by any neuron in the entire architecture.
That is, not only do we empirically expect that random neural networks have $n l$ knots, but also, we expect that $n l$ knots are \emph{created} among all neurons.

To proceed further with our analysis, it is now natural to shift our focus from the entire neural network to the individual neurons within.
As previously mentioned, with $l = 1$ hidden layer,~\eqref{eq:num-knots} is met almost surely: each neuron creates exactly one knot, and the probability that two neurons would create knots at the same location $x_j$ is zero for continuous distributions.
Yet, even with incrementally increasing $l$, empirical evidence still shows that~\eqref{eq:num-knots} is true to close approximation, if not exactly.
Thus, we can use induction to posit the following stronger statement.

\begin{conjecture}
    For the conditions in Lemma~\ref{lem:none-eliminated}, the expected number of knots created by every neuron $k = 1, \dots, n_i$ in every layer $i = 1, \dots, l$ is unity.
\end{conjecture}

From this conjecture, another immediately follows from the construction of the neuron.

\begin{conjecture}
    \label{conj:one-root}

    For the conditions in Lemma~\ref{lem:none-eliminated}, the expected number of roots of the affine transformation $\w_{ik} \cdot \v_{i-1}(x) + b_{ik}$, in every neuron $k = 1, \dots, n_i$ in every layer $i = 1, \dots, l$ (with $x$ in place of $\v_{i-1}(x)$ for $i = 1$), is unity.
\end{conjecture}

From this point on, our focus will lie primarily on Conjecture~\ref{conj:one-root} and the conditions for which it appears to be true.
Additionally, we will focus specifically on the affine transformation in~\eqref{eq:scalar-shallow}, which can be interpreted as the input into any neuron in the second hidden layer.
Analyses for further hidden layers may be pursued in future research.

\subsection{Equivalent form with forward-facing rectified linear units}
\label{sec:forward-facing}

Consider the single-layer scalar-valued neural network in~\eqref{eq:scalar-shallow} (e.g., \autoref{fig:relu_decomposition}(a)).
\begin{figure}[!t]
    \centering
    \begin{tikzpicture}
        \node at (-4.7, 0) {\includegraphics[scale=0.9]{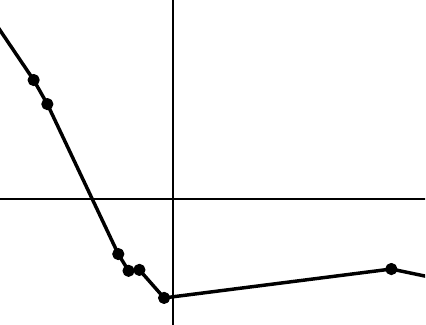}};
        \node at (0, 0) {\includegraphics[scale=0.9]{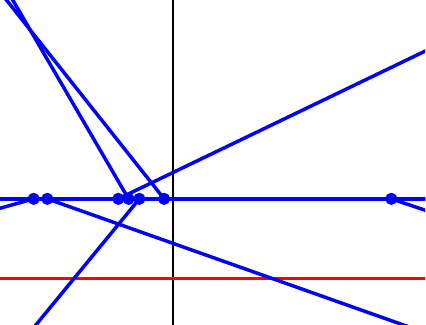}};
        \node at (4.7, 0) {\includegraphics[scale=0.9]{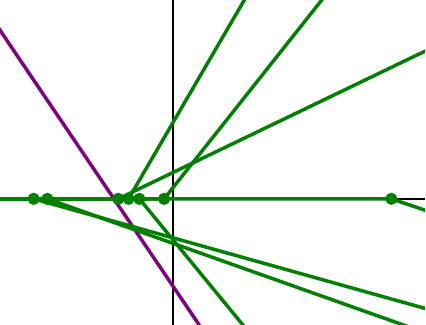}};

        \node at (-7.2, 1.5) [subfig] {(a)};
        \node at (-2.5, 1.5) [subfig] {(b)};
        \node at (2.2, 1.5) [subfig] {(c)};
    \end{tikzpicture}

    \caption{%
        (a)~A random single-layer neural network
        model~\eqref{eq:scalar-shallow}.
        (b)~The decomposition of (a) into individual rectified linear
        units (blue) and the bias $b_2$ (red).
        (c)~The equivalent model~\eqref{eq:forward-relu} with
        forward-facing rectified linear units (green) and a line
        (purple).%
    }
    \label{fig:relu_decomposition}
\end{figure}
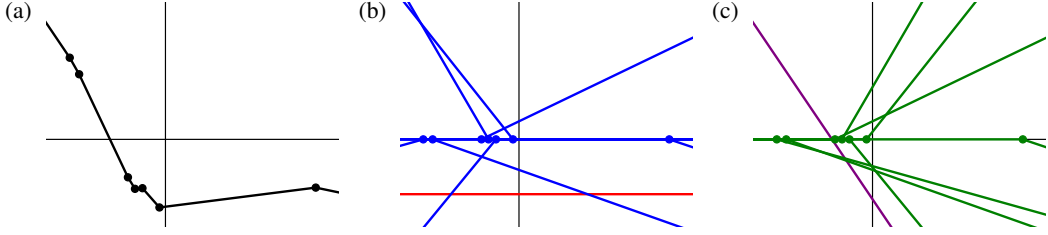
As discussed in \cite{ccr86994}, the activated ray in the rectified linear units---that is, the part of $\sigma(x)$ for $x \ge 0$---can be oriented toward any of the four quadrants of the $x$--$v_k$ plane (e.g., \autoref{fig:relu_decomposition}(b)).
For the purpose of analyzing the behavior of~\eqref{eq:scalar-shallow}, it is useful to have all rectified linear units activate in quadrants I or IV (e.g., \autoref{fig:relu_decomposition}(c)), so that rectified linear units can be successively turned on as $x$ is traced from $-\infty$ to $\infty$, and no rectified linear units are turned off.
The transformation that expresses the rectified linear units this way is presented as Lemma~1 of \cite{ccr86994}, which we restate here with slight modifications.

\begin{lemma}
    \label{lem:forward-relu}

    The neural network model~\eqref{eq:scalar-shallow} is equivalently
    \begin{equation}
        \label{eq:forward-relu}
        y(x) = \sum_{j=1}^n s_j \sigma(x - x_j) + c_1 x + c_0,
    \end{equation}
    where
    \begin{subequations}
        \label{eq:forward-relu-parameters}
        \begin{align}
            \label{eq:ci}
            c_1 &:= \sum_{\stack{1 \le j \le n}{w_{1j} < 0}} w_{2j} w_{1j}, &
            c_0 &:= \sum_{\stack{1 \le j \le n}{w_{1j} < 0}} w_{2j} b_{1j} + b_2, & \\
            \label{eq:xj}
            s_j &:= w_{2j} |w_{1j}|, &
            x_j &:= -\frac{b_{1j}}{w_{1j}}, &
            j &= 1, \dots, n.
        \end{align}
    \end{subequations}
    All rectified linear units in~\eqref{eq:forward-relu} face the positive $x$ direction (or ``forward'').
\end{lemma}

\begin{proof}
    The proof is given in \citet{ccr86994}.
\end{proof}

Besides making all rectified linear units in this transformation face forward, the utility of~\eqref{eq:forward-relu} is that the neural network is conveniently decomposed into easily understandable parts.
The coefficient $s_j$ describes the activated slope of the $j^{\rm th}$ rectified linear unit, for which the knot is located at $x = x_j$.
Furthermore, the coefficients $c_1$ and $c_0$ describe the line that must be added to the sum of rectified linear units to produce an equivalence with the original model~\eqref{eq:scalar-shallow}.
For the remainder of the paper, we will assume that $w_{1j}$, $b_{1j}$, and $w_{2j}$ are sorted so that $x_1 \le \dots \le x_n$.

Given~\eqref{eq:forward-relu}, it is natural to ask how the parameters~\eqref{eq:forward-relu-parameters} are correlated with each other, supposing that $w_{1j}$, $b_{ij}$, and $w_{2j}$ are independent for all $j = 1, \dots, n$, and the distributions are symmetric about zero.
Some pairs of parameters are uncorrelated by inspection.
For instance, consider $c_0$ and $c_1$.
These two parameters must have distributions that are also symmetric about zero.
Although they share $w_{2j}$ in common, their signs are independent because of the distinction between $w_{1j}$ in $c_1$ and $b_{1j}$ and $b_2$ in $w_{0j}$.
Therefore, with $E(\cdot)$ indicating the expected value, $E(c_0 c_1) = 0$, and $c_0$ and $c_1$ must be uncorrelated.
Likewise, each $s_j$ on its own is also uncorrelated from $c_0$, by the same argument of distributions being symmetric about zero.

For other pairs whose relations are less obvious, the correlation is most readily computed using the correlation coefficient.
We will not tabulate the correlation between each pair of parameters, but let us consider two examples.
First, the correlation between $s_j$ and $x_j$, by definition, is
\begin{subequations}
    \begin{align}
        \corr(s_j, x_j)
        &= \frac{E((s_j - E(s_j)) (x_j - E(x_j)))}{\sqrt{\var(s_j) \var(x_j)}} \\
        &= \frac{E(s_j x_j)}{\sqrt{\var(s_j) \var(x_j)}} \\
        &= \frac{E(-w_{2j} b_{1j} \sgn(w_{1j}))}{\sqrt{\var(s_j) \var(x_j)}} \\
        &= 0.
    \end{align}
\end{subequations}
Second, the correlation between each $s_j$ and $c_1$ is
\begin{subequations}
    \begin{align}
        \corr(s_j, c_1) &= \frac{E((s_j - E(s_j)) (c_1 - E(c_1)))}{\sqrt{\var(s_j) \var(c_1)}} \\
        &= \frac{E(s_j c_1)}{\sqrt{E(s_j^2) \var(c_1)}} \\
        \label{eq:corr-1}
        &= \frac{%
            E\left(w_{2j} |w_{1j}| \sum_{\stack{1 \le k \le n}{w_{1k} < 0}} w_{2k} w_{1k}\right)
        }{%
            \sqrt{E((w_{2j} w_{1j})^2) \var(c_1)}
        }.
    \end{align}
\end{subequations}
Since the weights are independent and their distributions are symmetric about zero, there is a $1 / 2$ chance that $w_{1j} < 0$ and hence the index $k = j$ appears in the sum.
The other half of the time, $k \ne j$ for all summands, and $w_{2j} |w_{1j}|$ is independent from the sum.
Thus,
\begin{subequations}
    \begin{align}
        \corr(s_j, c_1) &= \frac{E(-(w_{2j} w_{1j})^2)}{2 \sqrt{E((w_{2j} w_{1j})^2) \var(c_1)}} \\
        \label{eq:corr-2}
        &= -\frac{1}{2} \sqrt{\frac{E((w_{2j} w_{1j})^2)}{\var(c_1)}}.
    \end{align}
\end{subequations}
By the central limit theorem, $\var(c_1) \to \infty$ as $n \to \infty$, while $E((w_{2j} w_{1j})^2)$ remains fixed.
Thus, $\lim_{n \to \infty} \corr(s_j, c_1) = 0$.

\subsection{Probability distributions}
\label{sec:distributions}

In \autoref{sec:empirical}, numerical experiments showed that for $w_{1j}, b_{1j}, w_{2j}, b_2 \sim U(-1, 1)$ independently, the resulting neural network has a number of knots approximately equal to the number of neurons.
Conjecture~\ref{conj:one-root} then interpreted this to posit that the affine transformations in neurons are expected to have one root.
Here, we specifically examine the affine transformation~\eqref{eq:scalar-shallow} in the second-layer neurons for different probability distributions.

\subsubsection{Empirical root counts}
\label{sec:roots-distributions}

In \autoref{tab:distributions},
\begin{table}[!t]
    \centering
    \caption{%
        The number of roots of~\eqref{eq:scalar-shallow} for different weight and bias distributions, and in different domains, with $n = 10^4$.
        Each reported value is the average over $2\e{4}$ trials.
        A different set of trials is run for the two domains.%
    }
    \begin{tabular}{llll|ll}
        \multicolumn{4}{c|}{distributions} & \multicolumn{2}{c}{roots in} \\
        \hline
        $w_{1j}$ & $b_{1j}$ & $w_{2j}$ & $b_2$ & $\Reals$ & $(x_1, x_n)$ \\
        \hline\hline
        $N(0, 1)$ & $N(0, 1)$ & $N(0, 1)$ & $N(0, 1)$ & 0.9967 & 0.9918 \\
        $N(0, 1)$ & $N(0, 1)$ & $N(0, 1)$ & 0 & 0.9954 & 1.0073 \\
        $U(-1, 1)$ & $U(-1, 1)$ & $U(-1, 1)$ & $U(-1, 1)$ & 1.0064 & 0.9886 \\
        $U(-1, 1)$ & $U(-1, 1)$ & $U(-1, 1)$ & 0 & 0.9940 & 0.9967 \\
        $N(0, 1)$ & $N(0, 1)$ & $\{-1, 1\}$ & 0 & 0.9992 & 1.0015 \\
        $N(0, 1)$ & $N(0, 1)$ & $\{-1, 1\}$ & $N(0, 1)$ & 1.0064 & 1.0038 \\
        $U(-1, 1)$ & $U(-1, 1)$ & $\{-1, 1\}$ & 0 & 1.0081 & 1.0114 \\
        $U(-1, 1)$ & $U(-1, 1)$ & $\{-1, 1\}$ & $U(-1, 1)$ & 0.9904 & 0.9976 \\
        $N(0, 1)$ & $U(-1, 1)$ & $\{-1, 1\}$ & 0 & 1.0040 & 1.0133 \\
        $U(-1, 1)$ & $N(0, 1)$ & $\{-1, 1\}$ & 0 & 0.9975 & 1.0021 \\
        $\{-1, 1\}$ & $N(0, 1)$ & $N(0, 1)$ & 0 & 1.0018 & 0.8341 \\
        $\{-1, 1\}$ & $U(-1, 1)$ & $U(-1, 1)$ & 0 & 1.0022 & 0.6634 \\
        $N(0, 1)$ & $\{-1, 1\}$ & $N(0, 1)$ & 0 & 0.9975 & 1.0002 \\
        $U(-1, 1)$ & $\{-1, 1\}$ & $U(-1, 1)$ & 0 & 1.0086 & 1.0027 \\
        $N(0, 1)$ & $\{-1, 1\}$ & $\{-1, 1\}$ & 0 & 1.0044 & 0.9960 \\
        $U(-1, 1)$ & $\{-1, 1\}$ & $\{-1, 1\}$ & 0 & 1.0022 & 1.0070 \\
        $\{-1, 1\}$ & $N(0, 1)$ & $\{-1, 1\}$ & 0 & 1.0039 & 0.8464 \\
        $\{-1, 1\}$ & $U(-1, 1)$ & $\{-1, 1\}$ & 0 & 0.9929 & 0.6678
    \end{tabular}

    \label{tab:distributions}
\end{table}
we numerically compute the number of roots of~\eqref{eq:scalar-shallow}, for different probability distributions of the weights and biases.
In these experiments, we select the weights and biases independently, each from one of four common distributions: the standard normal $N(0, 1)$, the uniform distribution $U(-1, 1)$, Bernoulli trials $\{-1, 1\}$ with $1 / 2$ probability each, and the point mass at zero.
We then count the number of roots in $\Reals$ and in $(x_1, x_n)$ (i.e., between the most negative and the most positive knots only).
For each set of distributions, $2\e{4}$ trials are run on single-layer neural networks with $n = 10^4$ neurons, and the mean number of roots is reported.

The empirical number of roots is very simple: in all cases, the mean number of roots in all of $\Reals$ is roughly one, which is in excellent agreement with Conjecture~\ref{conj:one-root}.
This holds not only in very simple cases such as the last entry of \autoref{tab:distributions}---where knot locations and slopes are simply $x_j \sim U(-1, 1)$ and $s_j \sim \{-1, 1\}$, and $x_j$ and $s_j$ are independent---but also in more complicated cases where all weights and biases are normal or uniform.

It is also apparent from \autoref{tab:distributions} that if one were to search for roots only in $(x_1, x_n)$, then some of the simpler distributions produce a smaller number of roots in the affine transformation.
This scenario is less relevant to the study of random neural networks, since neurons can produce new knots for any $x \in \Reals$.
We analyze the presence of roots in $(-\infty, x_1)$ or $(x_n, \infty)$ in \autoref{sec:ends}.

\subsubsection{Scaling of distributions}
\label{sec:scaling}

Another notable property is that if the random distributions are symmetric about zero and $b_2 = 0$, then the scalings of the distributions have no effect on the expected number of roots.
This can be seen by examining the sets of weights and biases individually.
Most simply, if $w_{2j}$ were replaced by $a w_{2j}$ for some $a > 0$, then $y$ in~\eqref{eq:scalar-shallow} would be replaced by $a y$, and the number of roots would not change.
On the other hand, if the input biases $b_{1j}$ were replaced by $a b_{1j}$, then the model $y = \sum_j w_{2j} \sigma(w_{1j} x + a b_{1j})$ would equivalently be $y / a = \sum_j w_{2j} \sigma(w_{1j} x / a + b_{1j})$; that is, the entire model would be scaled both horizontally and vertically by $a$, again leaving the number of roots unchanged.
Finally, if the input weights $w_{1j}$ were replaced by $a w_{1j}$, then equivalently, $x$ in~\eqref{eq:scalar-shallow} would be replaced by $x / a$.
That is, the model would be horizontally inversely scaled, but the vertical scale and the number of roots would remain invariant.

At this point, it remains to describe how the distribution of the output bias $b_2$ affects the number of roots.
In \autoref{sec:linear}, we investigate the effect of scaling $c_0$.
By~\eqref{eq:ci}, such a scaling can be interpreted as a modification of only $b_2$, while all other parameters remain fixed.

\subsubsection{Example set of distributions}
\label{sec:example-distributions}

Next, we discuss a particular set of distributions that we will use for further empirical tests in \autoref{sec:integrated-random-walk}.
For the case of $w_{1j}, b_{1j} \sim N(0, 1)$, $w_{2j} \sim \{-1, 1\}$, and $b_2 = 0$, the distributions of the equivalent parameters in~\eqref{eq:forward-relu-parameters} can be derived.
The slopes $s_j = w_{2j} |w_{1j}|$ are given by $N(0, 1)$.
The factor of $|w_{1j}|$ is constrained to be positive, with zero probability of being zero.
On the other hand, $w_{2j}$ is chosen independently, so the sign of $s_j$ is positive or negative with one-half probability each.

Next, the knot locations $x_j = -b_{1j} / w_{1j}$ are given by the standard Cauchy distribution, which we denote $\cauchy(0, 1)$.
This distribution has a probability distribution function (PDF) $1 / (\pi (x^2 + 1))$ and cumulative distribution function (CDF) $(\tan^{-1} x) / \pi + 1 / 2$.
The Cauchy distribution is a bell curve, but its tails decay much more slowly than in normal distributions.
In fact, the slow decay causes the distribution to have undefined moments, including mean and variance.

Moving forward, we investigate the distributions of $c_0$ and $c_1$ These parameters are more challenging because the number of summands is not fixed, but is rather chosen from a binomial distribution.
Specifically, let $r$ be the number of terms in the set $\{w_{1j} | w_{1j} < 0\}$.
Since $w_{1j}$ is chosen symmetrically about zero, we let $B$ denote the binomial distribution and simply conclude that $r \sim B(n, 1 / 2)$.
Thus, in the limit of $n \to \infty$, $r \sim N(n / 2, n / 4)$.
For both $c_0$ and $c_1$, each of the $r$ summands is chosen from $N(0, 1)$: in the former case, the sign provided by $w_{2j}$ does not affect the distribution of $b_{1j}$; in the latter case, $w_{1j} < 0$, but $w_{2j}$ allows each summand to take either sign.

If $r$ is explicitly known, then $c_0, c_1 \sim N(0, r)$.
Therefore, the distributions of $c_0$ and $c_1$ for $n \to \infty$ can be computed by considering the distribution of the number of summands $r$.
Let $\phi$ and $\Phi$ respectively denote the PDF and CDF of the standard normal distribution.
Also, let $F_i$ denote the CDF of $c_i$ for $i = 0, 1$, and let $F_r$ denote the CDF of $r$.
We then find that for $n \to \infty$,
\begin{subequations}
    \label{eq:dist-0}
    \begin{align}
        dF_i(c_i)
        &= \int_{r=0}^\infty dF_i(c_i | r) dF_r(r) \\
        &= \int_{r=0}^\infty \left(\frac{1}{\sqrt{r}} \phi\left(\frac{c_i}{\sqrt{r}}\right) \, dc_i\right)
        \left(\frac{2}{\sqrt{n}} \phi\left(\frac{2 r - n}{\sqrt{n}}\right)\right) \, dr.
    \end{align}
\end{subequations}
Since the mean of $F_r$ scales with $n$, let us introduce a change of variables by defining $\hat{r} := r / n$, so as to remove this dependence.
Then,
\begin{equation}
    dF_i(c_i) = \int_{\hat{r}=0}^\infty \left(\frac{1}{\sqrt{n \hat{r}}} \phi\left(\frac{c_i}{\sqrt{n \hat{r}}}\right) \, dc_i\right)
    (2 \sqrt{n} \phi(\sqrt{n} (2\hat{r} - 1))) \, d\hat{r}.
\end{equation}
Note, however, that if $\delta$ denotes the Dirac delta function, then
\begin{equation}
    \lim_{n \to \infty} 2 \sqrt{n} \phi(\sqrt{n} (2\hat{r} - 1)) = \delta(\hat{r} - 1 / 2).
\end{equation}
Thus,
\begin{equation}
    \label{eq:dist-1}
    dF_i(c_i) = \sqrt{\frac{2}{n}} \phi\left(\sqrt{\frac{2}{n}} c_i\right) \, dc_i;
\end{equation}
that is, $c_0, c_1 \sim N(0, n / 2)$.
Further statistics may be derived; for example, $\corr(s_j, c_1) = - 1 / \sqrt{2 n}$, as per~\eqref{eq:corr-2}.

\section{Connection with integrated random walks}
\label{sec:integrated-random-walk}

In this section, we will draw and analyze a connection between the single-layer neural network architecture in~(\ref{eq:scalar-shallow},~\ref{eq:forward-relu}) and integrated random walks of variable step size.
The purpose of this connection is ultimately to provide a different perspective for understanding Conjecture~\ref{conj:one-root}.

First, the integrated random walk is defined and related to single-layer neural networks in \autoref{sec:equivalence}.
Next, we discuss the number of zero crossings of integrated random walks and some related statistics in \autoref{sec:zero-crossings}, and we comment on fixed versus variable step sizes.
In \autoref{sec:linear}, we empirically analyze how $c_0$ and $c_1$ in~\eqref{eq:forward-relu}---which are essentially the initial conditions of the associated integrated random walk---affect the number of zero crossings of the neural network.
Finally, we empirically study the variances and covariances of the integrated random walk in \autoref{sec:variance}.

\subsection{Overview and equivalence with random neural networks}
\label{sec:equivalence}

We will begin with basic definitions.
Note that because our goal is ultimately to relate integrated random walks with neural networks, we will use non-standard definitions and nomenclature.

\begin{definition}
    Let $s_j$ be i.i.d.\ random variables for $j = 1, \dots$.
    For some initial condition $y_0'$, a \emph{random walk} in one dimension is the sequence
    \begin{equation}
        \label{eq:random-walk}
        y_k' = y_0' + \sum_{j=1}^k s_j
    \end{equation}
    for $k = 0, 1, \dots$.
\end{definition}

\begin{remark}
    The random walk can be written in iterative form as
    \begin{equation}
        \label{eq:rw-iterative}
        y_{k+1}' = y_k' + s_{k+1}, \quad
        k = 0, 1, \dots.
    \end{equation}
    An example of a random walk is shown in \autoref{fig:random_walk}(a).
\end{remark}

\begin{figure}[!t]
    \centering
    \includegraphics{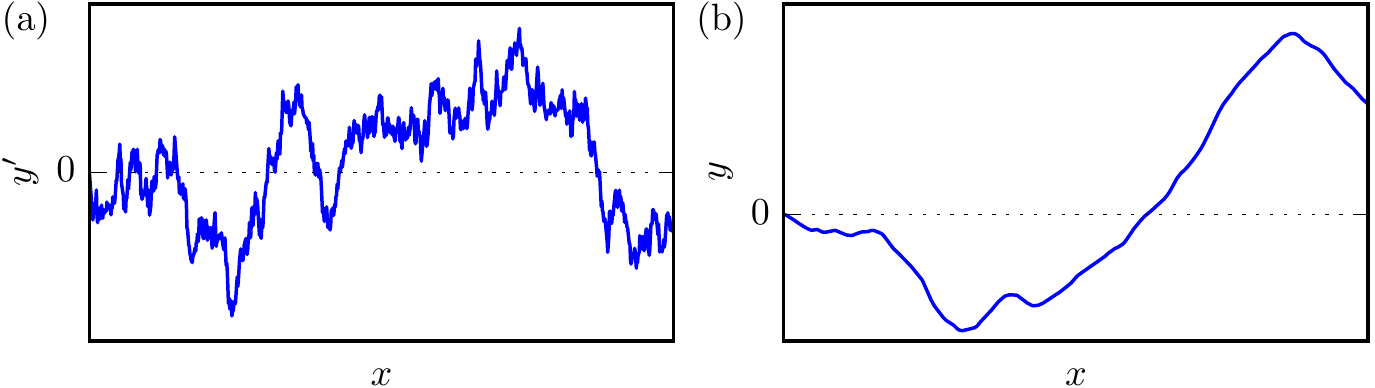}

    \caption{%
        (a)~A random walk~\eqref{eq:random-walk} with $10^3$ steps, where $s_j \sim N(0, 1)$ for all $j$.
        (b)~The integrated random walk~\eqref{eq:irw-fixed-dx} for (a).%
    }
    \label{fig:random_walk}
\end{figure}

\begin{definition}
    Given~\eqref{eq:random-walk} and some initial condition $y_1$, an \emph{integrated random walk of fixed step size} $\dx$ in one dimension is the sequence
    \begin{equation}
        \label{eq:irw-fixed-dx}
        y_k = y_1 + \dx \sum_{j=1}^{k-1} y_j'
    \end{equation}
    for $k = 1, \dots$.
\end{definition}

An example of an integrated random walk is shown in \autoref{fig:random_walk}(b).

\begin{definition}
    For $j = 1, \dots$, consider a sequence of points $x_j$ with $x_j < x_{j+1}$ for all $j$, and define
    \begin{equation}
        \label{eq:dx}
        \dx_j = x_{j+1} - x_j.
    \end{equation}
    Given~\eqref{eq:random-walk} and some initial condition $y_1$, an \emph{integrated random walk of variable step size} in one dimension is the sequence
    \begin{equation}
        \label{eq:integrated-random-walk}
        y_k = y_1 + \sum_{j=1}^{k-1} \dx_j \, y_j'
    \end{equation}
    for $k = 1, \dots$.
\end{definition}

\begin{remark}
    The integrated random walk of variable step size can be written in iterative form as
    \begin{subequations}
        \label{eq:irw-iterative}
        \begin{align}
            x_{k+1} &= x_k + \dx_k \\
            \begin{bmatrix}
                y_{k+1} \\ y_{k+1}'
            \end{bmatrix}
            &=
            \begin{bmatrix}
                1 & \dx_k \\ 0 & 1
            \end{bmatrix}
            \begin{bmatrix}
                y_k \\ y_k'
            \end{bmatrix}
            +
            \begin{bmatrix}
                0 \\ s_{k+1}
            \end{bmatrix}
        \end{align}
    \end{subequations}
    for $k = 1, \dots$, with initial conditions $y_1$ and $y_1' = y_0' + s_1$.
    With a fixed step size, the iteration is the same, except with $\dx$ in place of $\dx_k$.
    The notation is shown visually in \autoref{fig:irw-labeled}.
\end{remark}

\begin{figure}[!t]

\centering

\begin{tikzpicture}[x=7.5mm, y=7.5mm, node distance=1mm]
    \def\y{-0.15}

    \node (ic) at (0, 3) {};
    \node (1) at (1.5, 1.5) [knot] {};
    \node [above=of 1, blue] {$y_1$};
    \draw [gray] (1.5, \y) -- node [xtick] {$x_1$} (1);

    \node (2) at (3.45, 3.15) [knot] {};
    \node [above=of 2, blue, yshift=-1mm] {$y_2$};
    \draw [gray] (3.45, \y) -- node [xtick] {$x_2$} (2);

    \node (3) at (4.65, 1.8) [knot] {};
    \node [above=of 3, blue] {$y_3$};
    \draw [gray] (4.65, \y) -- node [xtick] {$x_3$} (3);

    \node (4) at (5.4, 2.55) [knot] {};
    \node [above=of 4, blue, yshift=-1mm] {$y_4$};
    \draw [gray] (5.4, \y) -- node [xtick] {$x_4$} (4);

    \node (5) at (7.5, 1.65) [knot] {};
    \node (f) at (8.5, 3.3) {};
    \node [above=of 5, blue, xshift=-1mm] {$y_5$};
    \draw [gray] (7.5, \y) -- node [xtick] {$x_5$} (5);

    \draw [thick] (0, 0) -- node [pos=1, xshift=-1mm, below] {$x$} (9, 0);
    \draw [thick] (0, 0) -- node [pos=1, yshift=-1mm, left] {$y$} (0, 3.75);

    \draw [thick, Green4] (ic) -- node [below, Green4, yshift=-1mm, pos=0.4] {$y_0'$} (1);
    \draw [thick, Green4] (1) -- node [below, Green4, pos=0.6] {$y_1'$} (2);
    \draw [thick, Green4] (2) -- node [below, Green4, pos=0.4, yshift=-1mm] {$y_2'$} (3);
    \draw [thick, Green4] (3) -- node [below, Green4, pos=0.6] {$y_3'$} (4);
    \draw [thick, Green4] (4) -- node [below, Green4] {$y_4'$} (5);
    \draw [thick, Green4] (5) -- node [right, Green4] {$y_5'$} (f);
\end{tikzpicture}



    \caption{The integrated random walk notation used in this section.}
    \label{fig:irw-labeled}
\end{figure}
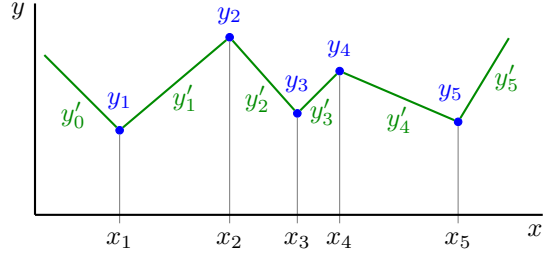

Having defined the integrated random walk, the equivalence with the single-layer neural network~\eqref{eq:forward-relu} can be shown.
For $k = 1, \dots, n$, (\ref{eq:random-walk}, \ref{eq:integrated-random-walk}, \ref{eq:irw-iterative}) is defined for discrete $x$ and $y$ and is merely a discrete sampling of~\eqref{eq:forward-relu}, which is defined for all $x \in \Reals$ and for continuous $y$.
Conceptually, both the forward-facing rectified linear unit form of the neural network~\eqref{eq:forward-relu} and the integrated random walk~\eqref{eq:integrated-random-walk} perform the same action.
That is, each time $x$ is advanced from one knot to the next, a new slope is added to the current slope.
This is achieved in the neural network by turning on another rectified linear unit, and in the integrated random walk by adding a new element of the random walk.
The equivalence is stated formally as follows.

\begin{theorem}
    \label{thm:equivalence}
    Let $x_k$ be given by~\eqref{eq:xj}, and let $y_k = y(x_k)$.
    The random single-layer neural network model~\emph{(\ref{eq:scalar-shallow},~\ref{eq:forward-relu})} and the integrated random walk of variable step size~\emph{(\ref{eq:random-walk},~\ref{eq:integrated-random-walk},~\ref{eq:irw-iterative})} produce the same set of points $(x_1, y_1)$, \ldots, $(x_n, y_n)$ if the initial conditions of the integrated random walk are
    \begin{subequations}
        \label{eq:ic}
        \begin{align}
            y_0' &= c_1 \\
            y_1 &= c_1 x_1 + c_0.
        \end{align}
    \end{subequations}
\end{theorem}

\begin{proof}
    First, note that under the assumptions of the random neural network model, the slopes $s_j$~\eqref{eq:xj} are i.i.d., as in the random walk.
    Now, substituting~(\ref{eq:random-walk},~\ref{eq:dx},~\ref{eq:ic}) into~\eqref{eq:integrated-random-walk} for $k = 1, \dots, n$ yields
    \begin{subequations}
        \begin{align}
            y_k &= c_1 x_1 + c_0 + \sum_{j=1}^{k-1} (x_{j+1} - x_j) \left(c_1 + \sum_{a=1}^j s_a\right) \\
            \begin{split}
                &= c_1 x_1 + c_0 + (x_k - x_1) c_1 + (x_2 - x_1) s_1 + (x_3 - x_2) (s_1 + s_2) + \dots \\
                &\quad + (x_k - x_{k-1})(s_1 + \dots + s_{k-1})
            \end{split} \\
            &= c_1 x_k + c_0 + (x_k - x_1) s_1 + (x_k - x_2) s_2 + \dots + (x_k - x_{k-1}) s_{k-1} \\
            \label{eq:irw-sum}
            &= c_1 x_k + c_0 + \sum_{j=1}^{k-1} s_j (x_k - x_j).
        \end{align}
    \end{subequations}
    Since $k > j$ in the summands, $x_k - x_j > 0$; thus, rectified linear units can be added with no effect, as
    \begin{equation}
        y_k = c_1 x_k + c_0 + \sum_{j=1}^{k-1} s_j \sigma(x_k - x_j).
    \end{equation}
    Furthermore, for $j = k, \dots, n$, $x_k - x_j \le 0$, so extra unactivated rectified linear units can be added, as
    \begin{equation}
        y_k = c_1 x_k + c_0 + \sum_{j=1}^n s_j \sigma(x_k - x_j).
    \end{equation}
    Thus, we have recovered~\eqref{eq:forward-relu} for $x = x_k$.
\end{proof}

The fact that random single-layer neural networks are equivalent to integrated random walks has an important impact for large $n$.
Most notably, in the limit that $n \to \infty$, most common choices of weight and bias distributions would lead to an increasingly dense set of knot points $x_j$.
If the knot points $x_j$ were to make up a continuum in some subset of $\Reals$, then the random walk~\eqref{eq:random-walk} would be a Wiener process---i.e., a Brownian motion---in $x$.
Therefore, the integrated random walk would be an integrated Wiener process.
More specifically, it would be a time-changed integrated Wiener process if the $x_j$ were not equally spaced.
Since the Wiener process is $C^0$ in $x$, the integrated Wiener process is $C^1$.
Thus, in the limit that $n \to \infty$, the random neural network must be reasonably smooth.
Even if $n$ were somewhat small, the random neural network would still be somewhat smooth in the sense of total variation.

The smoothness of random neural networks is evident
in \autoref{fig:example}, particularly in the right column, as well as
in \autoref{fig:random_walk}(b).
One of the most important
implications of the smoothness is that it prevents optimizers from
immediately overfitting neural networks to training data, particularly
when the data contain high-wavenumber noise.
We explore the nature of
early training in \autoref{sec:early-training}, where we examine
weight and bias distributions produced by optimizers.

\subsection{Zero crossings and initial conditions}
\label{sec:zero-crossings}

As previously stated, the objective of the integrated random walk analysis is to prove that our particular form of the integrated random walk has one root (Conjecture~\ref{conj:one-root}), so as ultimately to prove that the number of knots in random neural networks equals the number of neurons (Conjecture~\ref{conj:num-knots}).
Unfortunately, the number of roots in integrated random walks of fixed step size is still not very well understood \citep{DenisovAIHPPS15,GroeneboomAP99,KratzPS06,SinaiTMP92}, let alone the number of roots with variable step size.
Here, we review some of the previous literature with fixed step size and report some empirical findings, before turning to roots of integrated random walks with variable step size.

A number of studies have considered the semi-infinite integrated random walk of fixed step size with homogeneous initial conditions $y_0' = y_1 = 0$.
It has been shown that the probability that the first zero crossing---often called the exit time---occurs at or above some index $k$ scales asymptotically by $k^{-1/4}$.
This result was shown for the random variables $s_j \sim \{-1, 1\}$ by Sinai \citep{SinaiTMP92}, which proved the result both for a finite step size and for the integrated Wiener process, as the continuum limit of the integrated random walk.
This topic was further studied by a number of additional authors \citep[e.g.,][]{DenisovAIHPPS15,GroeneboomAP99}, who have refined the asymptotic limits on the exit time, computed the multiplicative coefficient on $k^{-1/4}$, and provided additional analytical insight.

At this point, however, there do not appear to exist exact (as opposed to asymptotic) analytical results on the zero crossings of integrated random walks.
We briefly report on some empirical findings based on $10^6$ trials of integrated random walks, with homogeneous initial conditions.
This model is equivalently
\begin{equation}
    \label{eq:irw-homogeneous}
    f(x) := \sum_{j=1}^n s_j \sigma(x - x_j).
\end{equation}
In each trial, the random variables are $s_j \sim N(0, 1)$, each chosen independently, and $10^6$ steps of fixed size $\dx = 1$ are taken.
Then, the properties of the zero crossings are evaluated.
\autoref{fig:survival}(a)
\begin{figure}[!t]
    \centering
    \includegraphics{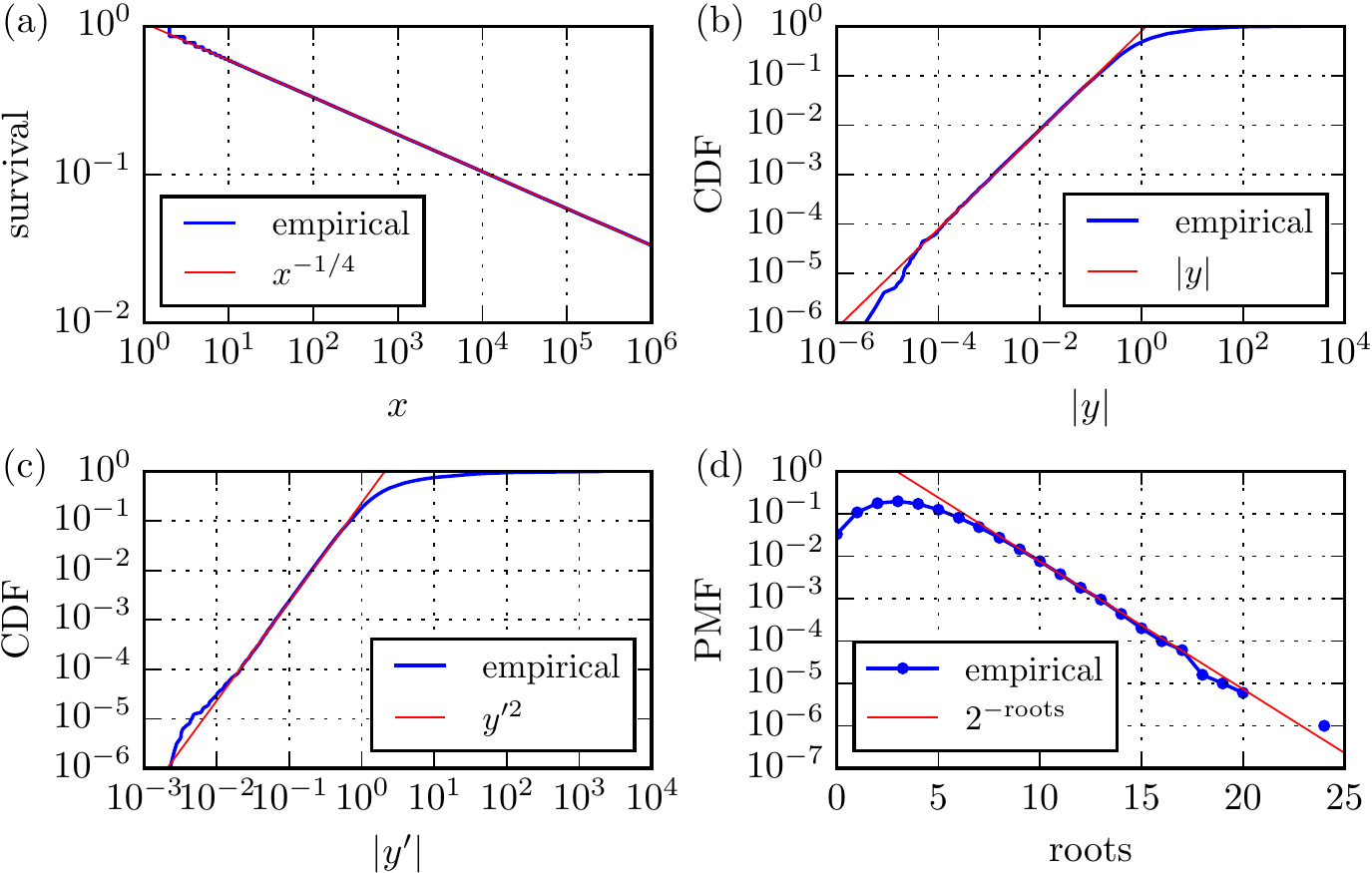}

    \caption{%
        Statistics of integrated random walks with $s_j \sim N(0, 1)$, with $10^6$ steps of fixed size $\dx = 1$.
        Empirical results (blue) and asymptotic limits (red) up to a multiplicative constant are shown.
        (a)~The survival function of the first zero crossing.
        (b)~The CDF of the integrated random walk value $|y|$ at the step before the first zero crossing.
        (c)~As (b), but with the random walk value $|y'|$.
        (d)~The probability mass function of the number of roots out of the $10^6$ steps.%
    }
    \label{fig:survival}
\end{figure}
shows the survival function (i.e., unity minus the CDF) of the first zero crossing's location.
The numerical experiments confirm the $x^{-1/4}$ asymptote \citep{DenisovAIHPPS15,GroeneboomAP99,SinaiTMP92}.

Next, \autoref{fig:survival}(b,~c) respectively shows the CDF of $|y|$ and $|y'|$ at the iteration just before the first zero crossing.
For $|y|, |y'| < 1$, the asymptotic scalings given by $|y|$ and $y'^2$ fit the data well.
Note that by comparison, for some $z \sim N(0, 1)$, the CDF of $|z|$ goes by $\sqrt{2 / \pi} |z| + \O(|z|^3)$ for $|z| \ll 1$.
The random walk and integrated random walk are of course more complicated, because the respective single and double integration of the random variable add a ``hysteresis'' effect.
For instance, if $y_k > 0$ and $y_k' \ll -1$, then a zero crossing is more likely to occur between indices $k$ and $k + 1$ than if $y_k' \approx -1$ or if $-1 \ll y_k' < 0$.
Thus, as reflected in \autoref{fig:survival}(c), large values of $|y'|$ before the first zero crossing are preferred over small values, as compared to the CDF $|z|$.

Finally, the probability mass function (PMF) of the number of roots, out of $10^6$ steps, is shown in \autoref{fig:survival}(d).
The median of the distribution is 3, and the empirical expected value is 3.71.
The most notable feature of the distribution is that the PMF is asymptotically related to the number of roots by $2^{-\mathrm{roots}}$.
Analytical results on the zero crossings of random walks are well-known.
For instance, for $s_k \sim \{-1, 1\}$, the probability that $y_k' = 0$ is asymptotically $1 / \sqrt{\pi k}$ for $k$ even and large, and the expected number of zeros out of $n$ steps is asymptotically $2 \sqrt{(n + 1) / \pi} - 1$ for $n$ even and large \citep{GrinsteadIP}.
In contrast, it is not presently clear for integrated random walks how the distribution of the number of zero crossings is analytically related to the number of steps.

In lieu of an analytical derivation, we empirically investigate the expected number of zero crossings as a function of the number of steps $n$.
In this numerical experiment, we simulate both integrated random walks with fixed $\dx$, as well as random neural networks~\eqref{eq:scalar-shallow} with $w_{1j}, b_{1j} \sim N(0, 1)$, $w_{2j} \sim \{-1, 1\}$, and $b_2 = 0$.
In the latter case, however, we then manually set $c_0 = c_1 = 0$ to be consistent with the homogeneous initial conditions in \autoref{fig:survival}.
For this particular case, as discussed in \autoref{sec:example-distributions}, the knot locations are given by $x_j \sim \cauchy(0, 1)$, and thus the step sizes $\dx_j$ are variable and random.
For both the integrated random walk and the neural network, the number of zero crossings is counted and averaged over 8,192 to $10^6$ trials, with $n = 1$ to $2^{20}$ steps.

\autoref{fig:crossings_ic}
\begin{figure}[!t]
    \centering
    \includegraphics{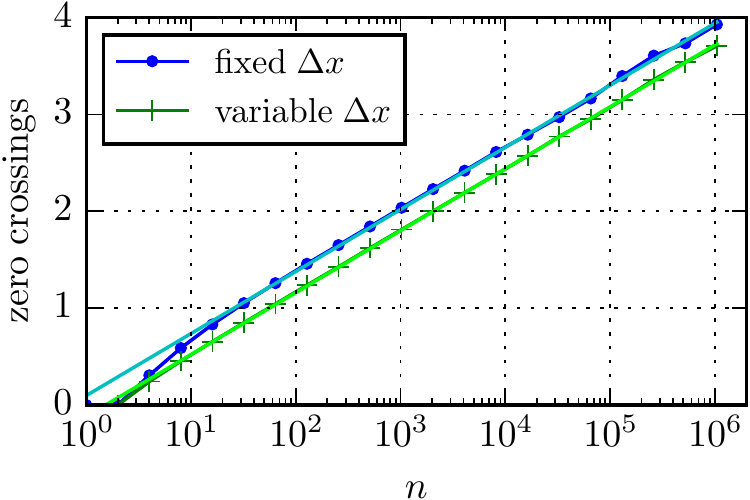}

    \caption{%
        Empirical number of zero crossings in $(x_1, x_n)$, for integrated random walks over $n$ steps with homogeneous initial conditions.
        Green: with fixed arbitrary step size $\dx$.
        Blue: as~\eqref{eq:scalar-shallow}, with $w_{1j}, b_{1j} \sim N(0, 1)$, $w_{2j} \sim \{-1, 1\}$, and $b_2 = 0$, but with $c_0$ and $c_1$ manually set to zero.
        For both, the best-fit lines for $n \gg 1$ are of the form $0.278 \ln n + \mathrm{const}$.%
    }
    \label{fig:crossings_ic}
\end{figure}
shows the dependence of the mean zero crossings on $n$.
A clear feature of the dependence is that for any $n$ that is reasonably large---about 16 for integrated random walks with fixed $\dx$, and 4 for the random neural networks with homogeneous initial conditions---the number of zero crossings is very close to $0.278 \ln n$ plus a constant.
The additive constants are different between the two cases.
Nevertheless, the fact that the two models have drastically different $x_j$ distributions, and yet have an essentially identical dependence on $n$, is unexpected.
We posit that the empirical $0.278 \ln n$ relationship may be supported by scaling arguments.

\subsection{Decomposition of random neural networks}
\label{sec:linear}

Another perspective on the roots of a random neural network can be gained by decomposing the neural network~\eqref{eq:forward-relu} into an integrated random walk with homogeneous initial conditions $y_0' = y_1 = 0$~\eqref{eq:irw-homogeneous} and the line $c_1 x + c_0$.
Quite simply, the random neural network $y(x)$ has a root wherever
\begin{equation}
    \label{eq:intersection}
    f(x) := \sum_{j=1}^n s_j \sigma(x - x_j) = -c_1 x - c_0.
\end{equation}
Whereas the integrated random walk form of the random neural network requires carefully chosen initial conditions~\eqref{eq:ic}, the utility of this decomposition is that it allows us to consider the simpler case of homogeneous initial conditions $y_0' = y_1 = 0$, as in \autoref{sec:zero-crossings}.
Instead of considering the roots of the integrated random walk, however, we now compare the integrated random walk to a line.

Three examples of the decomposition are shown in \autoref{fig:decomposition}.
\begin{figure}[!t]
    \centering
    \includegraphics{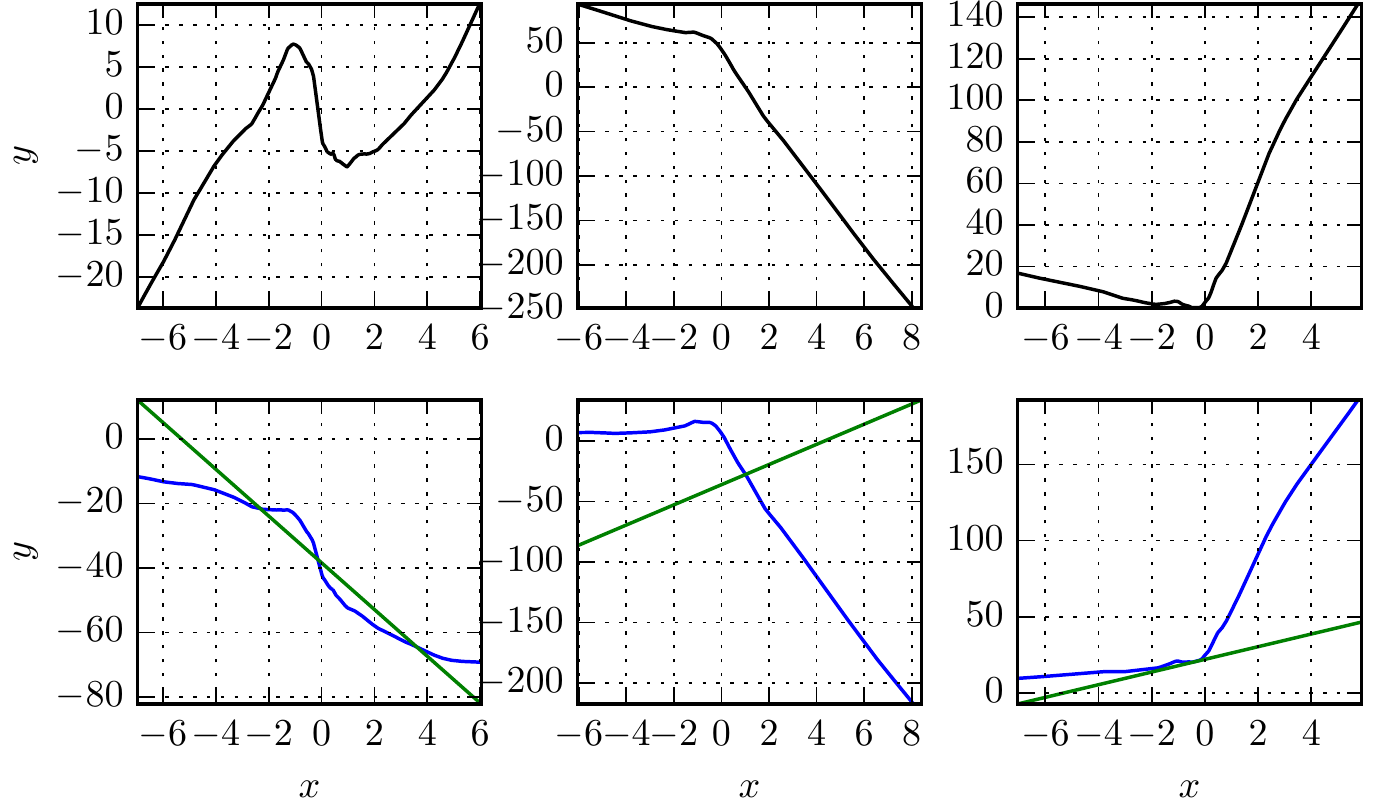}

    \caption{%
        Three examples (columns) of the decomposition of a random neural network with $n = 10^3$ neurons (top row) into an integrated random walk with homogeneous initial conditions (\eqref{eq:irw-homogeneous}; bottom row, blue) and the line $c_1 x + c_0$ (bottom row, green; shown as $-c_1 x - c_0$).
        Roots of the random neural network occur where the blue and green curves meet.%
    }
    \label{fig:decomposition}
\end{figure}
Again, Conjecture~\ref{conj:one-root} posited that random neural networks as described should have one root on average.
\autoref{fig:decomposition} shows examples with three, one, and zero roots.
By comparing the behavior of the integrated random walk and the line
over $x \in \Reals$, we may gain intuition on how many roots the
neural network should have.

The key experiment we conduct in this regard is to adjust the values of $c_0$ and $c_1$ manually after constructing random neural networks, and to observe how the expected number of roots is affected.
Here, we choose $w_{1j}, b_{1j} \sim N(0, 1)$, $w_{2j} \sim \{-1, 1\}$, and $b_2 = 0$ independently, as in \autoref{sec:example-distributions}; $n = 10^3$ neurons are used in each model.
Then, for each multiplicative factor in $2^{-20}, 2^{-19}, \dots, 2^{20}$, either $c_0$ or $c_1$ is multiplied by this factor while the other remains unchanged.
For each multiplicative factor, $10^5$ trials are run, and the average number of roots over the trials is computed.
The results of this experiment are shown in \autoref{fig:fudge}.
\begin{figure}[!t]
    \centering
    \includegraphics{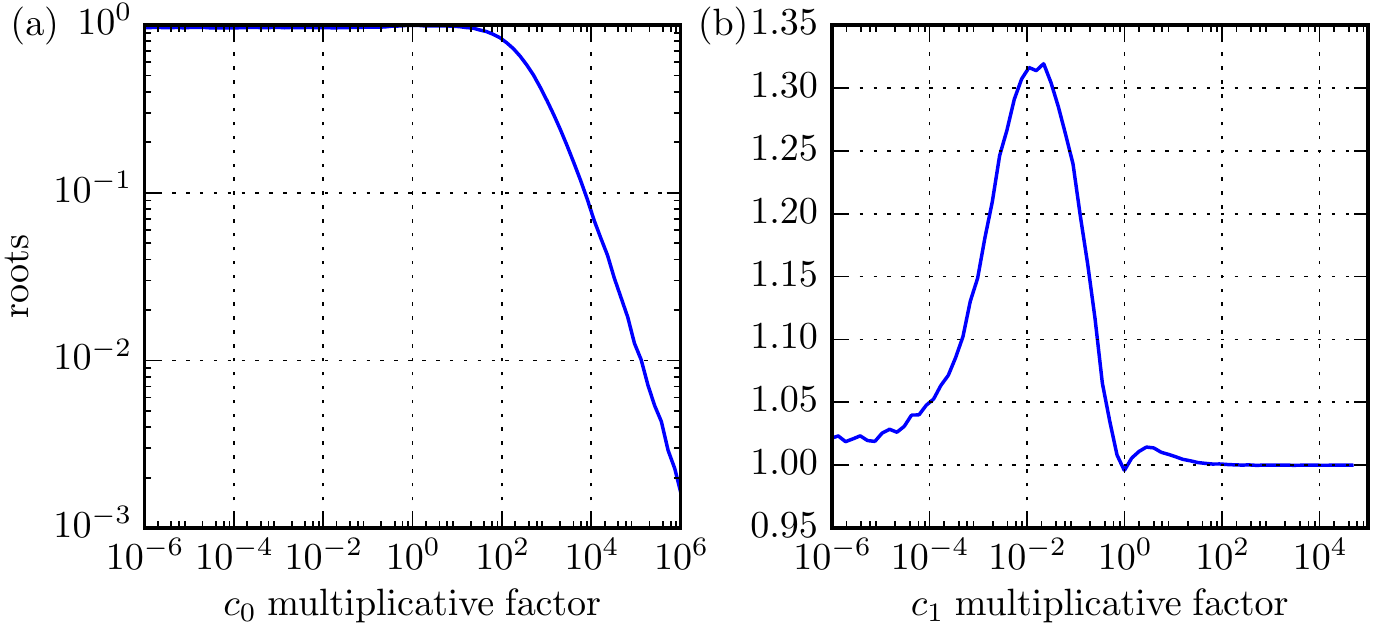}

    \caption{%
        The mean number of roots in the random neural network~\eqref{eq:forward-relu} with $n = 10^3$ as (a)~$c_0$ or (b)~$c_1$ is adjusted.%
    }
    \label{fig:fudge}
\end{figure}

Some of the features of \autoref{fig:fudge} are easily understandable, but others admittedly betray intuition.
For instance, if $|c_0|$ is increased greatly in \autoref{fig:fudge}(a), then the line $-c_1 x - c_0$ would be far removed from the integrated random walk (shown respectively in green and blue in \autoref{fig:decomposition}), thereby reducing the expected number of intersections.
A subtle phenomenon, however, is that if the multiplicative factor on $c_0$ is reduced from unity to zero, then the average number of roots actually decreases slightly---but statistically significantly---from 0.996 to 0.965.
It is not presently clear why the reduction of the bias would actually decrease the number of roots.

In \autoref{fig:fudge}(b), the expected number of roots approaches unity as the multiplicative factor on $c_1$ approaches infinity.
The reason is simply that in this limit, the right-hand side of~\eqref{eq:intersection} becomes a vertical line, which would have exactly one intersection with the integrated random walk.
Also, for $c_0 \ne 0$, a ``moderate'' multiplicative factor of around $10^{-2}$ on $c_1$ would produce a greater number of roots than a very small factor of less than $10^{-4}$.
With a very small factor, the line is flat and away from $y = 0$, and the integrated random walk would have reach $c_0$ to obtain an intersection.
The moderate factor also produces a greater number of roots than a very large factor.
With a moderate factor, the line stays near $y = 0$ for a large region of the domain, increasing the probability of intersections with the integrated random walk.
On the other hand, it is presently very unintuitive why the expected number of roots actually dips noticeably at a multiplicative factor of unity.
This suggests that the choice of $c_1$ given by~\eqref{eq:forward-relu-parameters} contains some special property that reduces the number of intersections between the integrated random walk and the line $-c_1 x - c_0$, as compared to both smaller and larger $c_1$.

A second set of experiments was conducted, where either $c_0$ or $c_1$
is multiplied by some factor, and the other coefficient is manually
set to zero.
In \autoref{fig:fudge_zeroed}(a),
\begin{figure}[!t]
    \centering
    \includegraphics{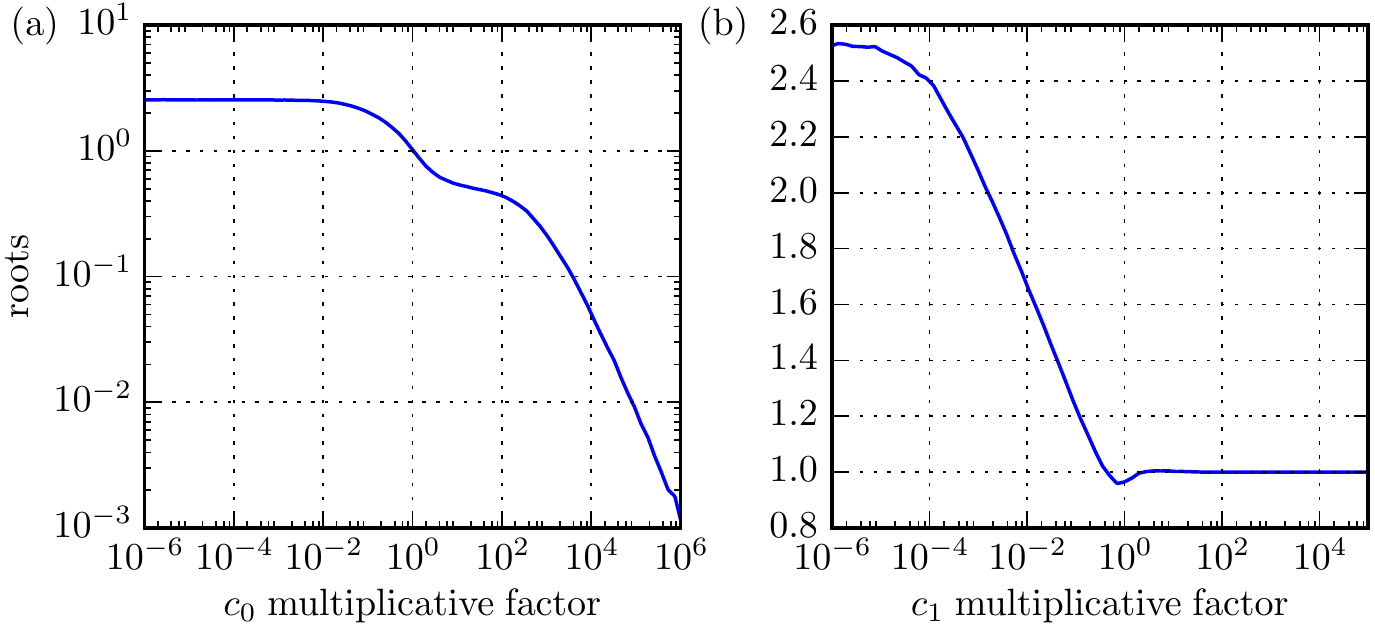}

    \caption{%
        The mean number of roots in the random neural network~\eqref{eq:forward-relu} with $n = 10^3$ as (a)~$c_0$ is adjusted and $c_1$ is set to zero, or (b)~$c_1$ is adjusted and $c_0$ is set to zero.%
    }
    \label{fig:fudge_zeroed}
\end{figure}
with $c_1 = 0$, small factors on $c_0$ increase the number of roots because a flat line near $y = 0$ is more likely to intersect with the integrated random walk than one further from $y = 0$.
In fact, the case with $c_0 = c_1 = 0$ corresponds exactly to root counting in the integrated random walk with variable step size and homogeneous initial conditions (\autoref{fig:crossings_ic}).
As a technicality, the case of $c_0 \to 0$ or $c_1 \to 0$ is different from $c_0 = 0$ or $c_1 = 0$, because the former does not allow the existence of a root between the first and second steps.

\autoref{fig:fudge_zeroed}(b) shows the number of roots with $c_0 = 0$ and $c_1$ adjusted by some multiplicative factor.
As before, a small $|c_1|$ allows the integrated random walk to intersect with the line more readily than a large $|c_1|$.
In the case of $|c_1| \gg 1$, the line is nearly vertical, and only one intersection is expected.
As with \autoref{fig:fudge}(b), however, it is unclear why there exists a small but statistically significant dip near a multiplicative factor of unity.
Again, there may exist some special property in the construction of $c_1$ that makes the number of roots increase of $c_1$ is altered in either direction.

\subsection{Variances and covariances}
\label{sec:variance}

Yet another angle from which Conjecture~\ref{conj:one-root} may be viewed is in terms of the variances and covariances of the random neural network, or of the decomposition described in \autoref{sec:linear}.
To begin, we turn again to the decomposition of the random neural network into an integrated random walk with homogeneous initial conditions and a line~\eqref{eq:intersection}.
Since the line is relatively easy to understand and the distribution of its coefficients can be derived, we focus on the variance of the integrated random walk with homogeneous initial conditions.
This is done with the understanding that the integrated random walk can later be compared back to the line to search for intersections of the two.

In the first experiment, $2\e{4}$ trials of random neural networks with $w_{1j}, b_{1j} \sim N(0, 1)$, $w_{2j} \sim \{-1, 1\}$, and $b_2 = 0$ (see \autoref{sec:example-distributions}) are run with $n = 10^3$ neurons.
In each trial, the line $c_0 x + c_1$ is discarded, and only the integrated random walk with homogeneous initial conditions is retained, as in~\eqref{eq:irw-homogeneous}.
Then, the points on the integrated random walk are separated into the $n$ Cauchy quantiles separated by the divisions at $x = \tan ((j / n - 1 / 2) \pi)$ for $j = 1, \dots, n - 1$.
For each quantile, the variance is computed and plotted in \autoref{fig:variance}.
\begin{figure}[!t]
    \centering
    \includegraphics{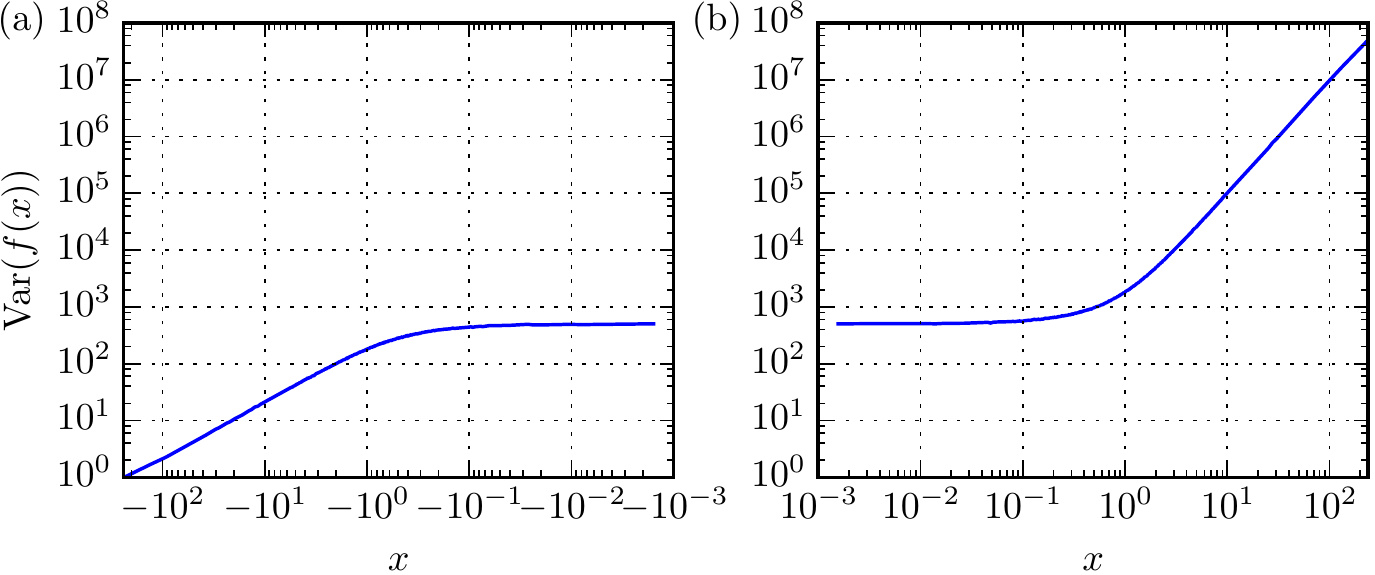}

    \caption{%
        Variance of the $10^3$-step integrated random walk with
        homogeneous initial conditions~\eqref{eq:irw-homogeneous}, for
        (a)~$x < 0$ and (b)~$x > 0$.%
    }
    \label{fig:variance}
\end{figure}

The data in \autoref{fig:variance} can be split roughly into three sections.
First, one quarter of the points $x_j \in \cauchy(0, 1)$ are in $(-\infty, -1)$, and the variance in this domain scales by $x^{-1}$.
Second, one half of the points $x_j$ are in $(-1, 1)$.
Although the variance looks flat in this region, the flatness is an artifact of the logarithmic scaling of the $x$-axis.
The empirical relation in this domain is roughly $\ln \var(f(x)) = 1.21 x + 6.26$, though higher-order terms are also present.
Finally, one quarter of the points $x_j$ are in $(1, \infty)$, and the variance scales by $x^2$.

The different scaling in $(-\infty, -1)$ and $(1, \infty)$ is very well-defined in the empirical data, and appears to a special feature of this integrated random walk.
By comparison, the variance of the integrated random walk of fixed step size~\eqref{eq:irw-fixed-dx} is known to scale by $\O(x^3)$: observing~\eqref{eq:irw-sum} with $c_0 = c_1 = 0$ and $x_k = (k - 1) \dx$ for some fixed $\dx$, we have that
\begin{equation}
    \var(y_k) = \dx \sum_{j=1}^{k-1} j^2 = (k^3 / 3 - k^2 / 2 + k / 6) \dx.
\end{equation}
The reasons for the differences in scaling laws are presently unknown.

Another quantity that can be analyzed is the correlation between a random neural network $y(x)$ and itself at a shifted location $x + h$.
The decay of the correlation with increasing $|h|$ can reveal the smoothness of $y(x)$, which may ultimately be tied back to the expected number of roots (see \citealt{AdlerICE} for a general reference).
Numerical experiments with the same weight and bias distributions as before, but with $n = 100$ neurons, are run with $10^4$ trials.
In this experiment, random neural networks are linearly interpolated onto the $n$ points dividing the $n + 1$ quantiles of the Cauchy distribution, and the correlation is computed as
\begin{equation}
    \corr(y(x), y(x + h)) = \frac{E(y(x) y(x + h))}{\sqrt{\var(y(x)) \var(y(x + h))}},
\end{equation}
since $E(y(x)) = 0$.

The result of this experiment is shown in \autoref{fig:correlation}.
\begin{figure}[!t]
    \centering
    \includegraphics{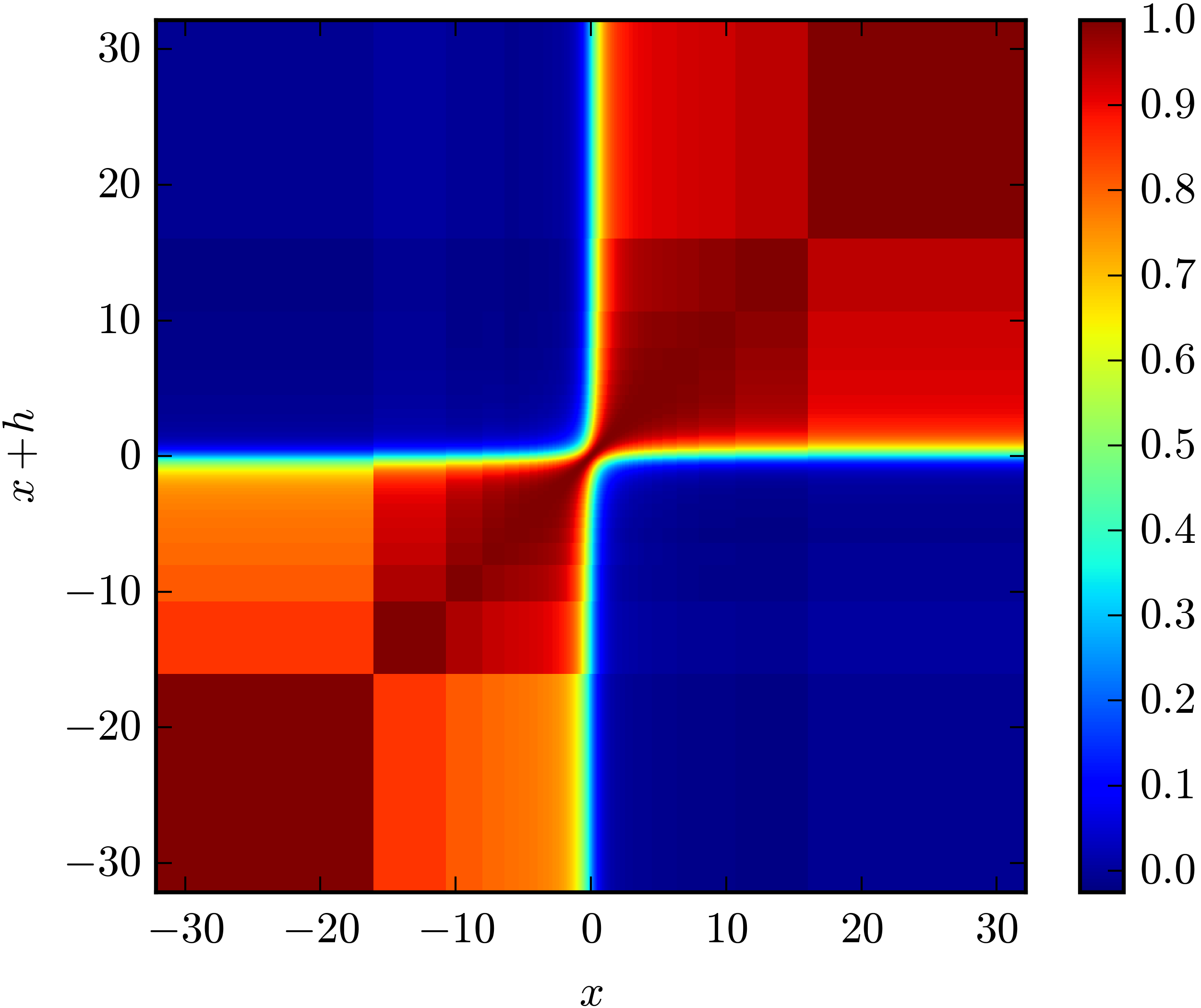}

    \caption{%
        The correlation coefficient between $y(x)$ and $y(x + h)$ for $n = 100$ neurons.%
    }
    \label{fig:correlation}
\end{figure}
A subset of this numerical result is also shown in \autoref{fig:correlation-h},
\begin{figure}[!t]
    \centering
    \includegraphics{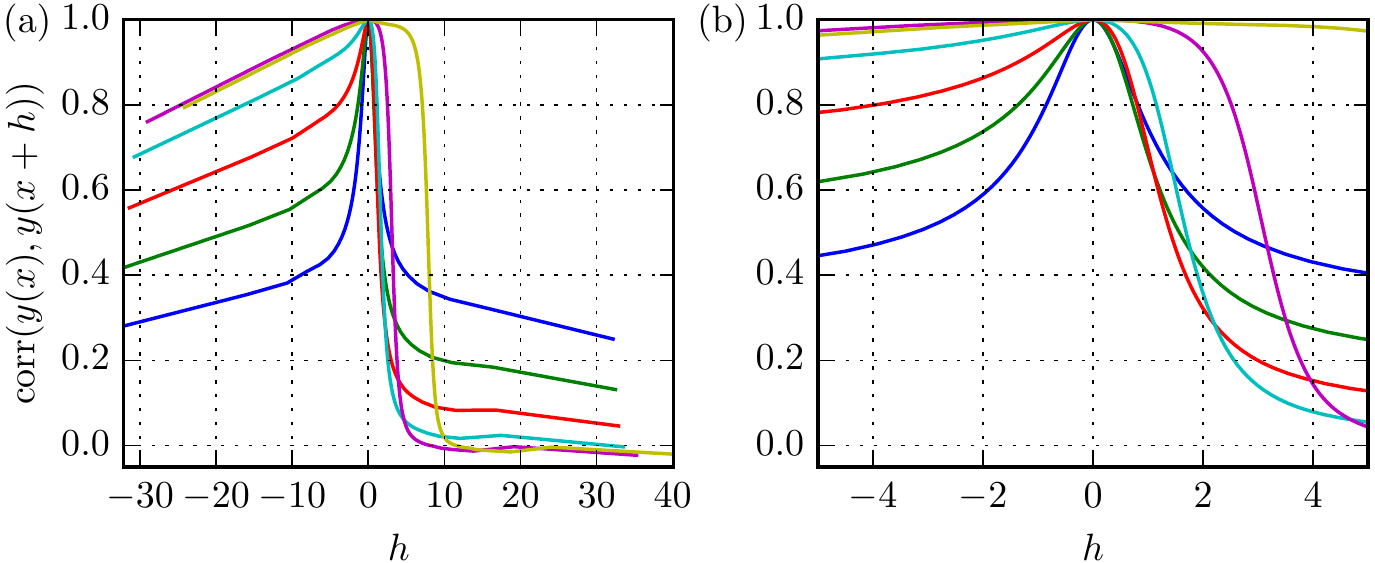}

    \caption{%
        The correlation between $y(x)$ and $y(x + h)$ for $x = -0.02$ (blue), $-0.34$ (green), $-0.74$ (red), $-1.39$ (cyan), $-3.11$ (magenta), and $-8.00$ (yellow).
        (a)~The entire domain.
        (b)~Zoomed near $h = 0$.%
    }
    \label{fig:correlation-h}
\end{figure}
which shows the correlation as a function of $h$ for different values of $x$.
In general, the random neural network at $x$ is notably correlated with itself at $x + h$ only if $\sgn(x) = \sgn(x + h)$.
We note as before that half of the knot points $x_j$ are in $(-1, 1)$, so the dynamics of the equivalent integrated random walk change very quickly near $x = 0$.
It is reasonable to posit that the relatively dense arrangement of knots in $(-1, 1)$ is enough to remove the correlation in $y$ between $x < 0$ and $x > 0$.
On the other hand, $y(x)$ is highly correlated with $y(x + h)$ for $x > 0$ and large $h$ (or for $x < 0$ and large $-h$) because the knot points are sparse for large $|x|$.

\section{Relation with early stages of training}
\label{sec:early-training}

\subsection{Zero bias}
\label{sec:zero_bias}

A choice of weight distributions effectively gives a measure on the function space parametrized by the neural network.
Therefore, exploring the structure of random networks under various weight distributions translates to statements about the output functions we expect to see from those networks.
The nature of random networks can be interpreted simply as a basic research problem, particularly with regard to the disparity between the potential complexity of the network and the empirically expected complexity, as given by the number of knots.

Yet, as we will later see, random networks also inform our understanding of actual training.
Standard network initialization methods produce weights which are either normal or uniform.
That is, in real networks, we expect to find weights distributed much as we see in the first several rows of \autoref{tab:distributions}.
Hence, our analysis almost immediately applies to actual training.

There remains, however, the question of the bias distributions.
Practically, the biases are almost always initialized to zero.
There are good reasons for this: first, the network is already made asymmetric by randomly selecting the weights, so there is no danger of identical nodes at the origin.
Second, just as with weights, there is no clear ``correct'' choice for initialization, so it may be better for the biases just to be zeroed.
Third, and probably most salient, such an initialization works, so there is little motivation to change.
We show in~\citet{WalkerSCAMP16} some of the consequences of zero bias initialization.
The primary one is that the network tends to fit the data closely near the origin: it learns from the origin outward.

Pathologies notwithstanding, it turns out that even with zero biases, most of the biases will be normally distributed after a tiny amount of gradient descent training, and the rest will be zero.
Although the variance of the normal part of the distribution is small, this result implies that when training real networks, we expect our description of random network structure to apply.

In this section, we explain the reason for the bias distribution, assuming that the neural network training is based on gradient descent, as is typically the case.
To show that the biases will have some distribution after one step of training, it suffices to compute the gradient of the loss function with respect to the biases.

\subsection{The gradient with respect to a bias}

To compute the gradient of the loss with respect to a bias, we first recall some notation for deep networks, as introduced in \autoref{sec:neural-nets}.
In layer $i = 1, \dots, l$, neuron $k = 1, \dots, n_i$ has a weight vector and bias respectively denoted by $\w_{i k}$ and $b_{i k}$.
Let the weight vector be $\w_{ik} = [w_{ik1} \; \cdots \; w_{ikn_{i-1}}]$.
Note that each scalar $w_{ikj}$ is the weight assigned by neuron $k$ of layer $i$ to the output value of neuron $j$ in layer $i-1$.
It will be important for us to understand the computed value at a neuron both before and after activation.
Using the notation in~\eqref{eq:deep}, let
\begin{equation}
    \vbar_{ik} =
    \begin{cases}
        \w_{ik} \cdot \x + b_{ik} &| \quad i = 1, \\
        \w_{ik} \cdot \v_{i-1} + b_{ik} &| \quad i > 1
    \end{cases}
\end{equation}
be the input to the rectified linear unit in neuron $k$ of hidden layer $i$, such that $v_{ik} = \sigma(\vbar_{ik})$ (cf.~\eqref{eq:deep-hidden}).
For simplicity, we will continue the assumption that the network has a single input and output.
As before, our analyses generalize immediately to multidimensional outputs.
It remains difficult to apply our analytical results to multidimensional inputs, but they still empirically hold.

Let us consider the typical loss functions given by the mean absolute error and the mean squared error.
If for a given input data point $x$, the desired output is $\hat{y}$ but the actual neural network output is $y$, then these losses are respectively $|y - \hat{y}|$ and $(y - \hat{y})^2$.
Their derivatives with respect to some bias $b_{ik}$ are
\begin{subequations}
    \begin{align}
        \p{|y - \hat{y}|}{b_{ik}} & = \sgn(y - \hat{y}) \frac{\partial y}{\partial b_{i k}}, \\
        \p{(y - \hat{y})^2}{b_{ik}} & = 2(y-\hat{y}) \frac{\partial y}{\partial b_{i k}}.
    \end{align}
\end{subequations}
In both cases, the gradient is a constant multiple of $\partial y / \partial b_{i k}$.
Thus, to understand the distribution of the gradient of the loss, it suffices to understand the gradient of the output.

Let $H$ denote the Heaviside step function
\begin{equation}
    H(x) :=
    \begin{cases}
        0 &| \quad x < 0 \\
        1 &| \quad x \ge 0
    \end{cases},
\end{equation}
which is the derivative of rectified linear unit $\sigma$.
Again, consider some bias $b_{ik}$.
Using the substitutions $j_{l+1} = 1$ and $j_i = k$, the chain rule yields
\begin{equation}
    \label{eq:gradient}
    \p{y}{b_{ik}} =
    \sum_{j_l=1}^{n_l} \sum_{j_{l-1}=1}^{n_{l-1}} \dots \sum_{j_{i+1}=1}^{n_{i+1}}
    \prod_{\alpha=i}^l w_{\alpha+1, j_{\alpha+1}, j_\alpha} H(\vbar_{\alpha j_\alpha});
\end{equation}
see \autoref{sec:gradient} for a derivation.
Although this expression looks complicated, it is conceptually straightforward.
This form iterates over all paths from neuron $k$ of layer $i$ to the output.
For each path, the weights along it and the associated Heaviside step functions are multiplied.
This entire expression is a function of the input $x$.

\subsection{The expected gradient}

In analyzing~\eqref{eq:gradient}, a pertinent question is how likely $H(\vbar_{i k}(x))$ is to be identically zero.
Because all the biases are zero at the start of neural network training, the form of $\vbar_{i k}(x)$ is actually quite simple: it must be a piecewise linear function with a single knot at the origin.
Thus, it looks like one of the four pictures shown in the top row of \autoref{fig:node_function_possibilities}.
\begin{figure}[!t]
    \centering
    \includegraphics{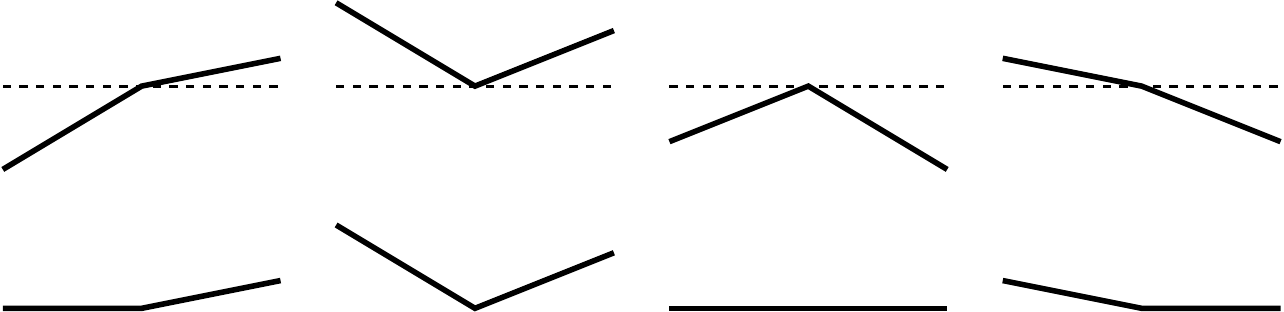}

    \caption{%
        The four possibilities for $\vbar_{i k}(x)$ (top) and $v_{i k}(x) = \sigma(\vbar_{i k}(x))$ (bottom) when all biases are zero.%
    }
    \label{fig:node_function_possibilities}
\end{figure}
The two slopes that define $\vbar_{i k}(x)$ are computed as linear combinations of the two slopes across the elements of the hidden layer output vector $\v_{i-1}(x)$.
To obtain $v_{i k}(x)$, we then apply the rectified linear unit to $\vbar_{i k}(x)$.
Hence, for many of the nodes, $\vbar_{i k}(x)$ is always negative and $v_{i k}(x)$ is therefore identically zero, as shown in the third column of \autoref{fig:node_function_possibilities}.
If the weights are independent and symmetric about zero, then this occurs for one quarter of the nodes in hidden layer $i = 2$.

In fact, as the layer index $i$ increases, the probability of an identically zero neuron grows.
The reason is that the left and right slopes of $\vbar_{i k}(x)$ are in fact correlated.
Understanding the precise correlation is a difficult problem, so it is hard to give the actual probability that $v_{i k}(x) = 0$ for all $x$.
Experimentally---and theoretically, in a simplified model---as the layer level $i$ increases, the probability that $v_{i k}(x) = 0$ for all $x$ approaches $1 / 2$.
Likewise, the probability that $v_{i k}$ is $\vee$-shaped also approaches $1 / 2$.

If $v_{i k}(x) = 0$ for all $x$, then the gradient $\partial y / \partial b_{i k} (x)$~\eqref{eq:gradient} will be zero.
By the argument above, this will apply to about half of the neurons as the layer level gets large.
On the other hand, if $v_{i k}(x)$ is not identically zero, then we can once again consider $\partial y / \partial b_{i k} (x)$.
With $l$ layers of $n$ neurons each, there exist $n^l$ paths from a neuron to the output.
Each of these paths produces a product in~\eqref{eq:gradient}.
Now, every node has a probability in $[1 / 4, 1 / 2)$ of being deactivated for all $x$, so the probability that a node is activated for some $x$ is at least $1 / 2$.
For an entire path to be nonzero for some $x$, all nodes along the path have to be activated, which occurs with probability at least $1/2^l$.
Hence, we expect more than $(n/2)^l$ of these products to be nonzero.
These are products of weights and Heaviside step functions, and are identically distributed.
Although they are not strictly independent, the correlations of the paths approaches zero as the path length increases.

Hence, we can apply the central limit theorem to draw our chief conclusion: the gradients $\partial y / \partial b_{i k}$ are random with a distribution given by the sum of a normal distribution and a point mass at zero.
The actual mass at zero, and a careful proof of many of the above assertions, are difficult to obtain and somewhat unnecessary, given how well the prediction is borne out in practice.

\section{Conclusion}
\label{sec:conclusion}

In this paper, we have analytically and numerically explored the knots of random scalar-input, scalar-output dense neural networks with rectified linear unit activations.
These random networks are encountered in early stages of training on data, as well as in extreme learning machines.
Numerical experiments show that the number of knots in $\Reals \to \Reals$ random neural networks is equal to the number of neurons, to very close approximation.
The same result would hold for $\Reals \to \Reals^p$ networks.
This number of knots empirically holds even when weights and biases are drawn from different combinations of the $N(0, 1)$, $U(-1, 1)$ and $\{-1, 1\}$ distributions.
The main thrust of this paper was to develop an analytical explanation of this behavior, and to comment on its implications.

Although we are unable at this point to develop a proof on the expected number of knots, we do develop certain equivalences that we believe may point toward an eventual proof.
In Lemma~\ref{lem:none-eliminated}, we showed that in the limit of $n_i \to \infty$ neurons in layer $i$, all knots from layer $i - 1$ are preserved in the layer $i$ outputs.
Therefore, we developed the central conjecture, Conjecture~\ref{conj:one-root}, which states that the input to every neuron is expected to have one root in this limit.
Again, this result is empirically supported, though a direct proof remains elusive.
To study this conjecture, we state in \autoref{thm:equivalence} that the input to every neuron is equivalently an integrated random walk with specific initial conditions.

In this paper, we also explored certain properties related to the general behavior of the neural network and the equivalent integrated random walk.
These properties include the correlation between integrated random walk parameters, as well as the scaling of the probability distributions.
Furthermore, we explored the decomposition of the integrated random walk into one with homogeneous initial conditions, and the straight line stemming from the nonzero initial conditions.
We then analyzed the zero crossings of integrated random walks with homogeneous initial conditions, the effects of changing the initial conditions, and the variances and covariances of the integrated random walks.

Next, we investigated the relation between random neural networks and the early stages of training.
The integrated random walk has certain properties such as smoothness in the continuum limit, which prevent optimizers from immediately overfitting to data.
We showed that although optimizers typically initialize biases to zero, the probability distribution of the biases quickly becomes a normal distribution plus a point mass at zero.
Thus, we verify that the random neural networks explored in this paper are in fact encountered in the real-life training of neural networks on data.

Recently, we have also analyzed different aspects of neural network theory.
In one companion paper \citep{ccr86994}, we derive a tight upper bound on the number of knots with the same neural network architecture we investigated in this paper.
In another \citep{WalkerSCAMP16}, we present empirical results on the relation between the size of neural networks and various parameters related to training on data.

Yet, many other important and relevant topics remain to be explored.
As mentioned above, a proof that the expected number of knots in a random neural network is equal to the number of neurons still remains to be found.
In addition, these $\Reals \to \Reals^p$ results need to be extended to $\Reals^q \to \Reals^p$ networks, for which the neural network divides the input space into convex polytopes, and the number of linear regions would have to be counted.
Furthermore, we would like to inquire if optimizers can be initialized differently to produce more complex networks in the early stages of training.
Such a task amounts to randomizing the neural network to be less smooth than $C^1$ in the continuum limit.
Although this increases the risk of overfitting, it may be advantageous if the function to be modeled is known to be complex.

Yet other questions in neural network theory can be explored.
Some works \citep[e.g.,][]{Raghu16} use the arclength of a neural network as an alternative measure of expressivity, and the arclength of random networks can be empirically measured.
In addition, we would like to obtain a practical measure of neural network expressivity by determining how well neural networks can fit functions in various Sobolev spaces.
This can be determined by constructing sums of basis functions with specified decay rates on the function coefficients, and empirically fitting neural networks to these sums.

Together, the ideas presented in the papers, along with the proposed research ideas, will help us form a better idea of neural network capabilities.
These ideas can help us construct a set of guidelines determining the network size that is necessary or sufficient for modeling specific functions.

\appendix

\section{Appendix: Roots beyond \texorpdfstring{$(x_1, x_n)$}{(x\_1, x\_n)}}
\label{sec:ends}

Let us consider the last entry of \autoref{tab:distributions}, in which we examine the roots of~\eqref{eq:scalar-shallow} with $w_{1j}, w_{2j} \sim \{-1, 1\}$, $b_{1j} \sim U(-1, 1)$, and $b_2 = 0$, all independently.
The discrepancy in the knots counts in $\Reals$ and in $(x_1, x_n)$ implies that the expected number of knots in $(-\infty, x_1) \cup (x_n, \infty)$ is greater than 0.
Upon examination of~\eqref{eq:forward-relu}, we see that the affine transformation in $(-\infty, x_1)$ is simply $c_1 x + c_0$, since none of the rectified linear units are activated.
Thus, a root appears in this interval if and only if the slope at $x < x_1$ has the same sign as $y(x_1)$---that is, $\sgn(c_1) = \sgn(c_1 x_1 + c_0)$, or equivalently,
\begin{equation}
    \label{eq:end-root-condition}
    x_1 + \frac{c_0}{c_1} > 0.
\end{equation}
In the limit that $n \to \infty$, $x_1 \to -1$, so a root appears in $(-\infty, x_1)$ if and only if $c_0 / c_1 > 1$.

The distributions for $c_0$ and $c_1$ can be computed by following the argument given in~(\ref{eq:dist-0}--\ref{eq:dist-1}) and employing the central limit theorem on $dF_i(c_i | r)$.
Since $\{-1, 1\}$ has unity variance and $U(-1, 1)$ has variance $1 / 3$, this argument shows that as $n \to \infty$, $c_0 \sim N(0, n / 6)$ and $c_1 \sim N(0, n / 2)$.
Numerical results agree with these estimates.
We demonstrated in \autoref{sec:forward-facing} that $c_0$ and $c_1$ are uncorrelated, but it can be further shown that they are independent for our choice of weight and bias distributions.
Thus, $c_0 / c_1 \sim \cauchy(0, 1 / \sqrt{3})$; that is, $c_0 / c_1$ has the CDF $\tan^{-1} (\sqrt{3} x) / \pi + 1 / 2$.
This CDF shows that $P(c_0 / c_1 > 1) = 1 / 6$.

Therefore, the probability that a root appears in $(-\infty, x_1)$ is $1 / 6$.
By symmetry arguments, it can also be shown that the probability that a root appears in $(x_n, \infty)$ is likewise $1 / 6$.
Hence, if empirically, one root is expected to appear in $(-\infty, \infty)$, then $2 / 3$ roots are expected to appear in $(x_1, x_n)$.
This is accurately reflected in \autoref{tab:distributions}.

Next, let $\phi$ and $\Phi$ respectively denote the PDF and CDF of the standard normal distribution.
In the penultimate entry of \autoref{tab:distributions}, where $w_{1j}, w_{2j} \sim \{-1, 1\}$ and $b_{1j} \sim N(-1, 1)$ independently, the most negative knot location $x_1$ is given by the first order statistic with probability density function $n \phi(x) (1 - \Phi(x))^{n-1}$.
The first order statistic must clearly decrease with increasing $n$, though the order statistic is sub-logarithmic in $n$, as follows.
The probability that $x_1$ is greater than some $x$ is $(1 - \Phi(x))^n$, so the CDF of $x_1$ is $1 - (1 - \Phi(x))^n$.
Thus, the median $\tilde{x}_1$ of $x_1$ is given by $1 - (1 - \Phi(\tilde{x}_1))^n = 1 / 2$, or
\begin{equation}
    \tilde{x}_1 = \Phi^{-1}(1 - 2^{-1 / n}).
\end{equation}
Using a series expansion for $n \to \infty$, it can be shown that retaining the leading orders,
\begin{equation}
    \tilde{x}_1 = -\sqrt{2 (\ln n - \ln \ln 2) - \ln(2 \pi) - \ln(2 (\ln n - \ln \ln 2) - \ln(2 \pi)))},
\end{equation}
which is dominated by $-\sqrt{2 \ln n}$.

On the other hand, we can follow a scaling argument reminiscent of the prior example and find that the variances of $c_0$ and $c_1$ both scale by $n$; hence, the distribution of $c_0 / c_1$ is independent of $n$.
Thus, as $n \to \infty$ increases, we expect that the probability of~\eqref{eq:end-root-condition} approaches zero, albeit at root-logarithmic rate.
Hence, the number of roots in $(x_1, x_n)$ will approach the number of roots in $(-\infty, \infty)$.
The convergence is extremely slow, however, and the limit cannot be reflected accurately in \autoref{tab:distributions}.

Lastly, we can see that in most of the entries of \autoref{tab:distributions}, nearly all the roots of~\eqref{eq:scalar-shallow} lie within $(x_1, x_n)$.
Differences between the average numbers of roots in $(-\infty, \infty)$ and $(x_1, x_n)$ are statistically indistinguishable.
This feature can once again be explained using scaling arguments.

Let us consider the case where $w_{1j}, b_{1j} \sim N(0, 1)$, $w_{2j} \sim \{-1, 1\}$, and $b_2 = 0$, all independently, as in \autoref{sec:example-distributions}.
It was shown there that $c_0, c_1 \sim N(0, n / 2)$.
Furthermore, the two parameters are independent; thus, $c_0 / c_1 \sim \cauchy(0, \sqrt{n / 2})$.
On the other hand, $x_j = - b_{1j} / w_{1j} \sim \cauchy(0, 1)$.
Denote by $g(x) = 1 / (\pi (x^2 + 1))$ and $G(x) = (\tan^{-1} x) / \pi + 1 / 2$ the PDF and CDF of the $\cauchy(0, 1)$ distribution, respectively.
Let us specifically consider $x \ll -1$, since smallest of $n$ samples from the $\cauchy(0, 1)$ distribution must almost certainly be much less than $-1$ for large $n$.
The PDF of $x_1$ and its series expansion for $x \ll -1$ are
\begin{subequations}
    \begin{align}
        n g(x) (1 - G(x))^{n - 1}
        &= \frac{n}{\pi (x^2 + 1)} \left(\frac{1}{2} - \frac{\tan^{-1} x}{\pi}\right)^{n-1} \\
        \label{eq:cauchy-order-stat}
        &= \frac{n}{\pi x^2} + \frac{n (n - 1)}{\pi^2 x^3} + \O(x^4).
    \end{align}
\end{subequations}
Also, observe that for some scaling factor $a > 0$,
\begin{equation}
    \frac{a^2 n}{\pi ((a x)^2 + 1)} \left(\frac{1}{2} - \frac{\tan^{-1} (a x)}{\pi}\right)^{a n-1}
    = \frac{n}{\pi x^2} + \frac{n (n - 1 / a)}{\pi^2 x^3} + \O(x^4),
\end{equation}
which differs from~\eqref{eq:cauchy-order-stat} by $\O(x^{-3})$.
Thus, as $n$ scales by $a$, $x_1$ must also scale by $a$---and therefore $x_1$ scales by $n$---to leading order.

As a side note, this could have been shown more quickly but less rigorously by considering only the median $\tilde{x}_1$ of $x_1$, as previously, instead of scaling its entire PDF\@.
The CDF of $\tilde{x}_1$ is $1 - (1 - ((\tan^{-1} x) / \pi + 1 / 2))^n$.
Setting this is equal to $1 / 2$, the median and its series expansion are
\begin{subequations}
    \begin{align}
        \tilde{x}_1 &= \tan\left(\pi\left(\frac{1}{2} - \frac{1}{2^{1 / n}}\right)\right) \\
        &= -\frac{n}{(\ln 2) \pi} - \frac{1}{2 \pi} + \O(n^{-1}),
    \end{align}
\end{subequations}
which scale by $n$ to leading order.

In all, $c_0 / c_1$ is centered at 0 and scales by $\sqrt{n / 2}$, but $x_1 \ll -1$ scales by $n$.
By this argument,~\eqref{eq:end-root-condition} is only satisfied with very small probability, even for modest $n$---with probability about 0.02 for $n = 50$---so almost all roots must reside within $(x_1, x_n)$.

\section{Appendix: Derivation of backpropagation gradient}
\label{sec:gradient}

Starting from~\eqref{eq:deep-output}, the scalar output of the neural
network is
\begin{subequations}
    \begin{align}
        y &= \w_{l+1, 1} \cdot \v_l + b_{l+1, 1} \\
        &= \sum_{j_l=1}^{n_l} w_{l+1,1,j_l} v_{l j_l} + b_{l+1,1}.
    \end{align}
\end{subequations}
Applying the chain rule, we have that
\begin{subequations}
    \begin{align}
        \p{y}{b_{ik}}
        &= \sum_{j_l=1}^{n_l} w_{l+1,1,j_l} \p{v_{l j_l}}{b_{ik}} \\
        &= \sum_{j_l=1}^{n_l} w_{l+1,1,j_l} H(\vbar_{l j_l}) \p{\vbar_{l j_l}}{b_{ik}}.
    \end{align}
\end{subequations}
Note that
\begin{subequations}
    \begin{align}
        \vbar_{l j_l} &= \w_{l j_l} \cdot \v_{l-1} + b_{l j_l} \\
        &= \sum_{j_{l-1}=1}^{n_{l-1}} w_{l j_l j_{l-1}} v_{l-1,j_{l-1}} + b_{l j_l};
    \end{align}
\end{subequations}
thus, continuing to apply the chain rule,
\begin{subequations}
    \begin{align}
        \p{y}{b_{ik}}
        &= \sum_{j_l=1}^{n_l} w_{l+1,1,j_l} H(\vbar_{l j_l})
        \sum_{j_{l-1}=1}^{n_{l-1}} w_{l j_l j_{l-1}} \p{v_{l-1,j_{l-1}}}{b_{ik}} \\
        &= \sum_{j_l=1}^{n_l} w_{l+1,1,j_l} H(\vbar_{l j_l})
        \sum_{j_{l-1}=1}^{n_{l-1}} w_{l j_l j_{l-1}} H(\vbar_{l-1,j_{l-1}}) \p{\vbar_{l-1,j_{l-1}}}{b_{ik}},
    \end{align}
\end{subequations}
and so on.
Finally, the chain rule ends with
\begin{equation}
    \p{v_{i j_i}}{b_{ik}} =
    \begin{cases}
        H(\vbar_{ik}) &| \quad j_i = k, \\
        0 &| \quad j_i \ne k
    \end{cases}.
\end{equation}
Therefore, the chain of sums continues to level $i + 1$, with level $i$ completing the chain rule with the term $w_{i+1, j_{i+1}, k} H(\vbar_{ik})$.
In all, we have that
\begin{equation}
    \begin{split}
        \p{y}{b_{ik}}
        &=
        \sum_{j_l=1}^{n_l} w_{l+1,1,j_l} H(\vbar_{l j_l})
        \sum_{j_{l-1}=1}^{n_{l-1}} w_{l j_l j_{l-1}} H(\vbar_{l-1, j_{l-1}}) \\
        &\quad \sum_{j_{l-2}=1}^{n_{l-2}} w_{l-1, j_{l-1}, j_{l-2}} H(\vbar_{l-2, j_{l-2}})
        \cdots \\
        &\quad \sum_{j_{i+1}=1}^{n_{i+1}} w_{i+2, j_{i+2}, j_{i+1}} H(\vbar_{i+1, j_{i+1}})
        w_{i+1, j_{i+1}, k} H(\vbar_{ik}).
    \end{split}
\end{equation}
By rearranging the sums and performing the appropriate substitutions, we obtain~\eqref{eq:gradient}.

\bibliography{ChenNIPS16}

\end{document}